%% file: paper.tex
\documentclass{article}
\usepackage{arxiv}

\input{preamble.tex}

\title{Even your Teacher Needs Guidance: \\ Ground-Truth Targets Dampen Regularization Imposed by Self-Distillation\thanks{An updated version of this paper is published under the same title at the 35th Conference on Neural Information Processing Systems (NeurIPS 2021).}}

\author{
  Kenneth Borup \\
  Department of Mathematics\\
  Aarhus University \\
  \texttt{kennethborup@math.au.dk} \\
  \and
  \textbf{Lars N. Andersen} \\
  Department of Mathematics\\
  Aarhus University \\
  \texttt{larsa@math.au.dk}
}

\begin{document}
\maketitle

\begin{abstract}
Knowledge distillation is classically a procedure where a neural network is trained on the output of another network along with the original targets in order to transfer knowledge between the architectures. The special case of self-distillation, where the network architectures are identical, has been observed to improve generalization accuracy. In this paper, we consider an iterative variant of self-distillation in a kernel regression setting, in which successive steps incorporate both model outputs and the ground-truth targets. This allows us to provide the first theoretical results on the importance of using the weighted ground-truth targets in self-distillation. Our focus is on fitting nonlinear functions to training data with a weighted mean square error objective function suitable for distillation, subject to $\ell_2$ regularization of the model parameters. We show that any such function obtained with self-distillation can be calculated directly as a function of the initial fit, and that infinite distillation steps yields the same optimization problem as the original with amplified regularization. Furthermore, we provide a closed form solution for the optimal choice of weighting parameter at each step, and show how to efficiently estimate this weighting parameter for deep learning and significantly reduce the computational requirements compared to a grid search.
\end{abstract}

\section{Introduction}
Knowledge distillation, most commonly known from \citet{hinton2015distilling}, is a procedure to transfer \textit{knowledge} from one neural network (teacher) to another neural network (student).\footnote{Often knowledge distillation is also referred to under the name \textit{Teacher-Student learning}.} Often the student has fewer parameters than the teacher, and the procedure can be seen as a model compression technique. Originally, the distillation procedure achieves the knowledge transfer by training the student network using the original training targets, denoted as ground-truth targets, as well as a softened distribution of logits from the (already trained and fixed) teacher network.\footnote{We will refer to the weighted outputs of the penultimate layer, i.e. pre-activation of the last layer, as logits.} Since the popularization of knowledge distillation by \citet{hinton2015distilling}, the idea of knowledge distillation has been extended to a variety of settings.\footnote{See Section \ref{sec:related_work} for a brief overview, or see \citet{wang2020knowledge} for a more exhaustive survey} This paper will focus on the special case where the teacher and student are of identical architecture, called self-distillation, and where the aim is to improve predictive performance, rather than compressing the model.

The idea of self-distillation is to use outputs from a trained model together with the original targets as new targets for retraining the same model from scratch. We refer to this as one step of self-distillation, and one can iterate this procedure for multiple distillation steps (see Figure \ref{fig:self_distill}). Empirically, it has been shown that this procedure often generalizes better than the model trained merely on the original targets, and achieves higher predictive performance on validation data, despite no additional information being provided during training \citep{furlanello2018born, Ahn2019VariationalTransfer, Yang2018KnowledgeStudents}. 
\begin{figure*}[htbp]
    \centering
    \includegraphics[width=0.9\linewidth]{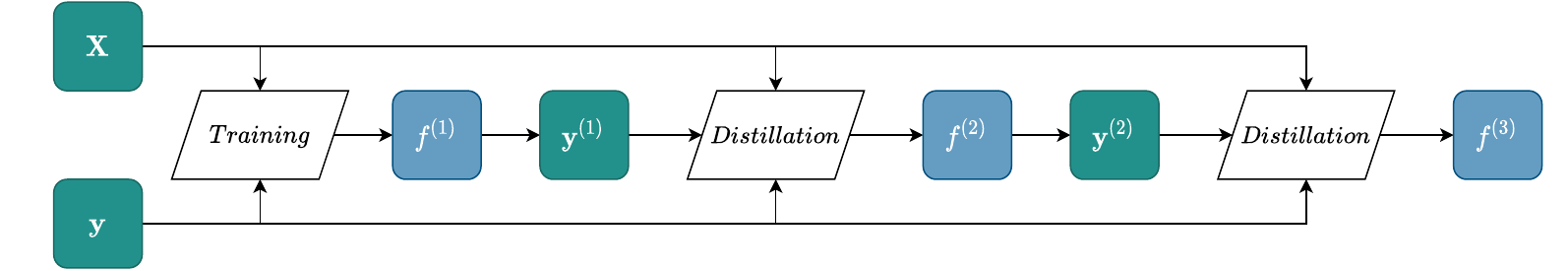}
    \caption{Illustration of self-distillation for two steps after the initial training, where we use the notation $f^{(\tau)} = f(\cdot, \hat{\bbeta}^{(\tau)})$. See Section \ref{sec:problem_setup} for details.}
    \label{fig:self_distill}
\end{figure*}

Modern deep neural networks are often trained in the over-parameterized regime, where the amount of trainable parameters highly exceed the amount of training samples. Under simple first-order methods such as gradient descent, such large networks can fit any target, but in order to generalize well, such overfitting is usually undesirable \citep{Zhang2017UnderstandingGeneralization, Nakkiran2020Deep}. Thus, some type of regularization is typically imposed during training, in order to avoid overfitting. A common choice is to add an $\ell_2$-regularization\footnote{With slight differences, $\ell_2$ regularization is often referred to as weight decay and ridge regularization in deep learning and statistical learning literature, respectively.} term to our objective function, which has been shown to perform comparably to early-stopping gradient descent training \citep{Yao2007OnLearning}. However, in the theoretical study of the over-parameterized regime, regularization is often overlooked, but recent results have shown a connection between wide neural networks and kernel ridge regression through the Neural Tangent Kernel (NTK) \citep{Lee2019WideDescent, Lee2020GeneralizedNetworks, Hu2019SimpleGuarantee}. We briefly elaborate on this connection in Section \ref{sec:ntk}, which motivates our problem setup and connection to deep learning in Section \ref{sec:optimal_alpha_DL}.

\section{Related Work}\label{sec:related_work}
The idea of knowledge distillation dates back to \citet{caruana2006model}, and was later brought to the deep learning setting by \citet{ba2013deep} and more recently popularized by \citet{hinton2015distilling} in the context of compressing neural networks.
Since the original formulation, various extensions have been proposed. Some approaches focus on matching the teacher and student models on statistics other than the distribution of the logits, such as intermediate representations \citep{romero2014fitnets}, spacial attention maps \citep{zagoruyko2016paying}, Jacobians \citep{srinivas2018knowledge}, Gram matrices \citep{yim2017gift}, or relational information between teacher outputs \citep{park2019relational}. Other extensions focus on developing the transfer procedure, such as self-distillation \citep{furlanello2018born}, data-free distillation \citep{lopes2017datafree, nayak2019zeroshot, micaelli2019zeroshot, chen2019datafree, fang2019data}, data distillation \citep{radosavovic2017data}, residual knowledge distillation \citep{gao2020residual}, online distillation \citep{anil2018large} or contrastive distillation \citep{Ahn2019VariationalTransfer, Tian2020Contrastive}.

The practical benefits of knowledge distillation have been proven countless of times in a variety of settings, but the theoretical justification for knowledge distillation is still highly absent. \citet{hinton2015distilling} conjecture that the success of knowledge distillation should be attributed to the transfer of \textit{dark knowledge} (e.g. inter-class relationships revealed in the soft labels). \citet{muller2019does, tang2020understanding}
support this conjecture, and argue that knowledge distillation is similar to performing adaptive label smoothing weighted by the teacher's confidence in the predictions. \citet{Dong2019DistillationNetwork} show the importance of early stopping when training over-parameterized neural networks for distillation purposes by arguing that neural networks tend to fit informative and simple patterns faster than noisy signals, and knowledge distillation utilizes these simple patterns for knowledge transfer. \citet{Abnar2020TransferringDistillation} empirically investigate how knowledge distillation can transfer inductive biases between student and teacher models, and \citet{Gotmare2019ADistillation} empirically shows how the dark knowledge shared by the teacher mainly is disbursed to some of the deepest layers of the teacher.

To the best of our knowledge, few papers investigate knowledge distillation from a rigorous theoretical point of view, and those that do, do so with strong assumptions on the setting. \citet{phuong2019towards} ignore the ground-truth targets during distillation and furthermore assume linear models. \citet{mobahi2020selfdistillation} investigate self-distillation in a Hilbert space setting with kernel ridge regression models where the teacher is trained on the ground-truth targets, and the student (and subsequent iterations) is only trained on the predictions from the teacher without access to the ground-truth targets. They show that self-distillation progressively limits the number of basis functions used to represent the solutions, thus eventually causing the solutions to underfit.
In this paper, we build on the theoretical results of \citet{mobahi2020selfdistillation}, but we include the weighted ground-truth targets in the self-distillation procedure, where we allow the weight to depend on the self-distillation step, and show how this drastically affects the behavior and effect of self-distillation.\footnote{In Supplementary Material \ref{app:mobahi} we relate our problem setup to \citet{mobahi2020selfdistillation} and extend some of our results to a constrained optimization setting with a regularization functional in Hilbert space.}

\paragraph{Our Contributions}
Through a theoretical analysis we show that
\begin{itemize}
    \item the solution at any distillation step can easily be calculated as a function of the initial fit, and infinitely many steps of self-distillation (with fixed distillation weight) correspond to solving the usual kernel ridge regression problem with a specific amplified regularization parameter for non-zero weights,
    \item for fixed distillation weights, self-distillation amplifies the regularization at each distillation step, and the ground-truth targets dampen the sparsification and regularization of the self-distilled solutions, ensuring non-zero solutions for any number of distillation steps,
    \item the optimal distillation weight has a closed form solution for kernel ridge regression, and can be estimated efficiently for neural networks compared to grid search.
\end{itemize}

The proofs of all our results can be found in Supplementary Material \ref{sec:proofs}, and code to reproduce our example in Section \ref{sec:illustrative_example} and results in Section \ref{sec:experiments} can be found at \href{https://github.com/Kennethborup/self_distillation}{github.com/Kennethborup/self\_distillation}.

\paragraph{Notation}
Vectors and matrices are denoted by bold-faced letters; vectors are column vectors by default, and for a vector $\ba$ let $[\ba]_i$ be the $i$-th entry and for a matrix $\bA$ let $[\bA]_{i,j}$ be the $(i,j)$-th entry. Let $\bI_n$ denote the identity matrix of dimension $n$, $[k] = \{1,2, \dots, k\}$, and let $\norm{\cdot}_2$ and $\norm{\cdot}_F$ denote the $\ell_2$-norm and the Frobenius norm, respectively. Finally, for a function $h: \R^n \to \R^d$ and $\bX \in \R^{m \times n}$, we denote by $h(\bX)$ the $\R^{m \times d}$ matrix of outcomes, where the $i$'th row of $h(\bX)$ is the function applied to the $i$'th row of $\bX$, i.e. $[h(\bX)]_{i, \cdot} = h(\bx_i)$.

\section{Problem Setup}\label{sec:problem_setup}
Consider the training dataset $\D \subseteq \R^{d} \times \R$, and let $\X = \{\bx \mid (\bx, y) \in \D\}$ and $\Y = \{y \mid (\bx, y) \in \D\}$ denote the inputs and targets, respectively. Let $\bX = [\bx_i]_{i \in [n]} \in \R^{n \times d}$ be the matrix of inputs, $\by = [y_i]_{i \in [n]}$ the vector of targets, and $\tilde{\bX} \in \R^{m \times d}, \tilde{\by} \in \R^m$ be the matrix and vector of validation inputs and targets, respectively. Given a feature map $\varphi : \R^{d} \to \mathcal{V}$, where $\mathcal{V}$ has dimension $D$, we denote by $\bK = \kappa(\bX, \bX) = [\kappa(\bx_i, \bx_j)]_{i,j=1}^{n} \in \R^{n \times n}$, where $\kappa(\bx_i, \bx_j) = \langle \varphi(\bx_i), \varphi(\bx_j)\rangle$, the symmetric kernel (Gram) matrix associated with the feature map $\varphi$.\footnote{Since the kernel trick makes the predictions depend only on inner products in the feature space, it is not a restriction if $D$ is infinite. However, for ease of exposition we assume $D$ is finite.}

\subsection{Self-Distillation of Kernel Ridge Regressions}\label{sec:sd_krr}
In order to avoid overfitting our training data, we will impose a regularization term on our weights, and thus investigate the kernel ridge regression functions $f \in \F$ mapping $f: \X \to \Y$, to construct a solution which best approximates the true underlying data generating map and generalize well to new unseen data from this underlying map. We consider self-distillation in the kernel ridge regression setup; i.e. consider the (self-distillation) objective function
\begin{align}\label{eq:distill_objective}
    &\L^{\mathrm{distill}}(f(\bX, \bbeta), \by_{1}, \by_{2}) = \frac{\alpha}{2}\norm{f(\bX, \bbeta) - \by_{1}}_2^2 + \frac{1 - \alpha}{2}\norm{f(\bX, \bbeta) - \by_{2}}_2^2 + \frac{\lambda}{2}\norm{\bbeta}_2^2,
\end{align}
where $\alpha \in [0,1]$, $\lambda > 0$, $\by_{1}, \by_{2} \in \R^n$ and $f(\bX, \bbeta) = \varphi(\bX)\bbeta$. The objective in \eqref{eq:distill_objective} is a weighted sum of two Mean Square Error (MSE) objective functions with different targets\footnote{It is easy to verify that minimizing \eqref{eq:distill_objective} and the classic MSE objective with a weighted target, i.e. $\tilde{\L}^{\mathrm{distill}}(f(\bX, \bbeta), \by_{1}, \by_{2}) = \frac{1}{2}\norm{f(\bX, \bbeta) - (\alpha\by_{1} + (1-\alpha)\by_2)}_2^2 + \frac{\lambda}{2}\norm{\bbeta}_2^2$, are equivalent and that the objective functions are equal up to the additive constant $\alpha(\alpha - 1)\norm{\by_1 - \by_2}_2^2$.} and an $\ell_2$-regularization on the model weights. Minimization of \eqref{eq:distill_objective} w.r.t. $\bbeta$ is straightforward and yields the minimizer
\begin{align}
    \hat{\bbeta} &\eqdef \argmin_{\beta} \L^{\mathrm{distill}}(f(\bX, \bbeta), \by_{1}, \by_{2}) \label{eq:unconstrained_optimization}
    = \varphi(\bX)^\trans \left(\bK + \lambda\bI_n \right)^\inv \left(\alpha\by_1 + (1-\alpha)\by_2 \right)
\end{align}
by Woodbury's matrix identity and definition of $\bK$. This solution can also be seen as a direct application of the Representer Theorem \citep{scholkopf2001generalized}. Let $\by^{(0)} \eqdef \by$, i.e. the original targets, and recursively define for the steps $\tau \geq 1$,
\begin{align}
    \hat{\bbeta}^{(\tau)} &\eqdef \argmin_{\bbeta} \L^{\mathrm{distill}}(\bbeta, \by, \by^{(\tau-1)}) \label{eq:beta_minimizer} \\
    &= \varphi(\bX)^\trans \left(\bK + \lambda\bI_n \right)^\inv \left(\alpha^{(\tau)}\by + (1-\alpha^{(\tau)})\by^{(\tau-1)} \right), \nonumber \\
    f(\bx, \hat{\bbeta}^{(\tau)}) &\eqdef \varphi(\bx)^\trans\hat{\bbeta}^{(\tau)} \label{eq:f_tau}\\
    &= \kappa(\bx, \bX)^\trans\left(\bK + \lambda\bI_n \right)^\inv \left(\alpha^{(\tau)}\by + (1-\alpha^{(\tau)})\by^{(\tau-1)} \right), \nonumber\\
    \by^{(\tau)} &\eqdef f(\bX, \hat{\bbeta}^{(\tau)}) \label{eq:y_tau},
\end{align}
for fixed $\alpha^{(\tau)} \in [0,1]$.
Notice, the initial step ($\tau = 1$) corresponds to standard training by definition and as such is independent of $\alpha^{(1)}$. Self-distillation treats the weighted average of the predictions, $\by^{(1)}$, from this initial model on $\bX$, and the ground-truth targets, $\by$ as targets. This procedure is repeated as defined in \eqref{eq:beta_minimizer}-\eqref{eq:y_tau} and we obtain the self-distillation procedure as illustrated in Figure \ref{fig:self_distill}. Note, the special cases $\alpha^{(\tau)} = 0$ and $\alpha^{(\tau)} = 1$ correspond to merely training on the predictions from the previous step, and only training on the original targets, respectively. Thus, $\alpha^{(\tau)} = 1$ is usually not of interest, as the solution is equal to a classical kernel ridge regression, and self-distillation plays no role in this scenario. We will often consider the special case of equal weights, $\alpha^{(2)} = \dots = \alpha^{(\tau)} = \alpha$, and if $\alpha = 0$ this corresponds to the setting investigated in \citet{mobahi2020selfdistillation} in a slightly different setup. Thus, some of the following results can be seen as a generalization of \citet{mobahi2020selfdistillation} to step-wise and non-zero $\alpha$.

\section{Main Results}\label{sec:main_results}
In this section we present our main results for finitely and infinitely many distillation steps along with a closed form solution for the optimal $\alpha^{(\tau)}$ as well as an illustrative example highlighting the effect of the chosen sequence of $(\alpha^{(t)})$ on the solutions.

\subsection{Finite Self-Distillation Steps}
Our first result, which follows from straightforward computations, states that the predictions obtained after any finite number of distillation steps can be expressed directly as a function of $\by$ and the kernel matrix $\bK$ calculated at the initial fit ($\tau = 1$).
 
\begin{theorem}\label{thm:unroll_Yt}
Let $\by^{(\tau)}, \hat{\bbeta}^{(\tau)}$, and $f(\cdot, \hat{\bbeta}^{(\tau)})$ be defined as above. Fix $\alpha^{(2)}, \dots, \alpha^{(\tau)} \in [0, 1)$, and let $\eta(i, \tau) \eqdef \prod_{j=i}^{\tau} \left(1-\alpha^{(j)} \right)$, then for $\tau \geq 1$, we have that
\begin{align}
    &\by^{(\tau)} = \left( \sum_{i=2}^{\tau} \alpha^{(i)} \eta(i+1, \tau) \left(\bK\left(\bK + \lambda \bI_n\right)^{-1}\right)^{\tau-i+1} + \eta(2, \tau)\left(\bK\left(\bK + \lambda \bI_n\right)^{-1}\right)^{\tau} \right) \by, \label{eq:recurrent_y} \\
    &f(\bx, \hat{\bbeta}^{(\tau)}) = \alpha^{(\tau)} f(\bx, \hat{\bbeta}^{(1)}) + (1-\alpha^{(\tau)})f(\bx, \hat{\bbeta}^{(\tau)}_{\alpha=0}) \label{eq:recurrent_f}
\end{align}
for any $\bx \in \R^{d}$, where $\hat{\bbeta}^{(\tau)}_{\alpha=0}$ is the minimizer in \eqref{eq:beta_minimizer} with $\alpha^{(\tau)} = 0$.
\end{theorem}

Since \eqref{eq:recurrent_y} and \eqref{eq:recurrent_f} are expressed only in terms of $\bK$, $(\bK + \lambda \bI_n)^\inv$, and $\kappa(\bx, \bX)$ we are able to calculate the predictions for the training data as well as for any $\bx \in \R^{d}$ based merely on the initial fit ($\tau = 1$) without the need for any additional fits. Hence, despite the calculations of $\bK$, $\kappa(\bx, \bX)$, and especially $(\bK + \lambda \bI_n)^\inv$ being (potentially) highly computationally demanding, when obtained, we can calculate any distillation step directly by the equations in Theorem \ref{thm:unroll_Yt}. Furthermore, predictions at step $\tau$ can be seen as a weighted combination of two classical ridge regression solutions, based on the original targets and the predicted targets from step $\tau-1$, respectively. However, choosing appropriate $\alpha^{(t)}$ for $t = 2, \dots, \tau$ is non-trivial. We explore these dynamics in Section \ref{sec:optimal_alpha} and \ref{sec:infinite_steps}. First, we use Theorem \ref{thm:unroll_Yt} to analyse the regularization that self-distillation progressively impose on the solutions.

\subsection{Effective Sparsification of Self-Distillation Solutions}\label{sec:sparsification}
We now show that we can represent the solutions as a weighted sum of basis functions, and that this basis sparsifies when we increase $\tau$, but also that the amount of sparsification depends on the choice of $\alpha$. A similar sparsification result for the special case of fixed $\alpha^{(\tau)} = 0$ for $\tau \geq 1$ was proved in \citet{mobahi2020selfdistillation}, and in particular, our \eqref{eq:fx_basis} generalizes equation (47) in their paper.

Using the spectral decomposition of the symmetric matrix $\bK$ we write $\bK = \bV \bD \bV^\trans$, where $\bV \in \R^{n \times n}$ is an orthogonal matrix with the eigenvectors of $\bK$ as rows and $\bD \in \R^{n \times n}$ is a non-negative diagonal matrix with the associated eigenvalues in the diagonal. Inserting the diagonalization yields
\begin{align}
    \bK(\bK + \lambda\bI_n)^{-1} &= \bV\bD\bV^\trans(\bV\bD\bV^\trans + \lambda\bI_n)^{-1} \\
    &= \bV \bD \left(\bD + \lambda\bI_n\right)^{-1}\bV^\trans, \label{eq:kernel_with_SVD}
\end{align}
where $\lambda > 0$.
By straightforward calculations using \eqref{eq:recurrent_y} and \eqref{eq:kernel_with_SVD} we have
\begin{align}\label{eq:Y_from_B}
    \by^{(\tau)} &= \bV \bB^{(\tau)} \bV^{\trans} \by, \quad \text{where}\\
    \bB^{(\tau)} &\eqdef \sum_{i=2}^{\tau} \alpha^{(i)} \eta(i+1, \tau) \bA^{\tau-i+1} + \eta(2, \tau)\bA^{\tau}, \quad \text{and} \quad \bA \eqdef \bD(\bD + \lambda\bI_n)^{-1} \label{eq:Bt},
\end{align}
and $\bA$, $\bB^{(\tau)} \in \R^{n \times n}$ are diagonal matrices for any $\tau$.
Furthermore, by \eqref{eq:Y_from_B} the only part of the solution depending on $\tau$ is the diagonal matrix, $\bB^{(\tau)}$, and in the following we show how $\bB^{(\tau)}$ determines the effective sparsification of the solution $f(\cdot, \hat{\bbeta}^{(\tau)})$.

\begin{lemma}\label{lem:recurrent_B}
Let $\bB^{(\tau)}$, and $\bA$ be defined as above, and let $\bB^{(0)} \eqdef \bI_n$. Then we can express $\bB^{(\tau)}$ recursively as
\begin{align}\label{eq:recurrent_B}
    \bB^{(\tau)} = \bA\left((1-\alpha^{(\tau)})\bB^{(\tau-1)} + \alpha^{(\tau)} \bI_n\right),
\end{align}
and $[\bB^{(\tau)}]_{k,k} \in [0,1]$ is (strictly) decreasing in $\tau$ for all $k \in [n]$ and $\tau \geq 1$ if $\alpha^{(2)} = \dots = \alpha^{(\tau)} = \alpha$.
\end{lemma}

Similarly to \eqref{eq:Y_from_B}, if we use Lemma \ref{lem:recurrent_B} and Theorem \ref{thm:unroll_Yt}, we can show that for any $\bx \in \R^p$
\begin{align}
    f(\bx, \hat{\bbeta}^{(\tau)}) &= \kappa(\bx, \bX)^\trans \bV \bD^\inv \bB^{(\tau)} \bV^\trans \by \nonumber \\
    &= \bp(\bx)^\trans \bB^{(\tau)} \bz, \quad \text{where} \label{eq:fx_basis}\\
    \bp(\bx) &\eqdef \bD^\inv \bV^\trans \kappa(\bx, \bX), \quad \text{and} \quad \bz \eqdef \bV^\trans \by. \nonumber
\end{align}
Thus, the solution $f(\cdot, \hat{\bbeta}^{(\tau)})$ can be represented as a weighted sum of some basis functions, where the basis functions are the components of the orthogonally transformed and scaled basis $\bp(\bx)$, and $\bz$ is an orthogonally transformed vector of targets.

Now assume $\alpha^{(2)} = \dots = \alpha^{(\tau)} = \alpha$ for any $\tau \geq 2$ for the remaining of this section. In the following we show how $\bB^{(\tau)}$ effectively sparsifies with each distillation step when $\alpha < 1$, and thus also effectively sparsifies the solution $f(\cdot, \hat{\bbeta}^{(\tau)})$.
Lemma \ref{lem:recurrent_B} not only provides a recursive formula for $\bB^{(\tau)}$, but also shows that each diagonal element of $\bB^{(\tau)}$ is in $[0,1]$ and is strictly decreasing in $\tau$, which in turn implies that the self-distillation procedure progressively shrinks the coefficients of the basis functions. Using Lemma \ref{lem:recurrent_B} we can now show, that not only does $\bB^{(\tau)}$ decrease in $\tau$, smaller elements of $\bB^{(\tau)}$ shrink faster than larger elements, as we elaborate on below the theorem.

\begin{theorem}\label{thm:B_sparsification}
For any pair of diagonals of $\bD$, i.e. $d_k$ and $d_j$, where $d_k > d_j$, we have for all $\tau \geq 1$,
\begin{align}\label{eq:B_diagonal_evolution}
    \frac{[\bB^{(\tau)}]_{k,k}}{[\bB^{(\tau)}]_{j,j}} &= \begin{cases} \frac{1 + \frac{\lambda}{d_j}}{1 + \frac{\lambda}{d_k}}, &\text{for } \alpha = 1, \\ \left(\frac{1 + \frac{\lambda}{d_j}}{1 + \frac{\lambda}{d_k}}\right)^\tau, &\text{for } \alpha = 0, \end{cases}
\end{align}
and if we let $\mathrm{sgn}(\cdot)$ denote the sign function\footnote{Note, we use the definition of $\mathrm{sgn}(\cdot)$ where $\mathrm{sgn}(0) \eqdef 0$.}, then for $\alpha \in (0, 1)$ we have that
\begin{align}
    &\mathrm{sgn}\left(\frac{[\bB^{(\tau)}]_{k,k}}{[\bB^{(\tau)}]_{j,j}} - \frac{[\bB^{(\tau-1)}]_{k,k}}{[\bB^{(\tau-1)}]_{j,j}} \right) \nonumber\\
    &\quad= \mathrm{sgn}\left(\left( \left(\frac{[\bB^{(\tau-1)}]_{k,k}}{[\bB^{(\tau-1)}]_{j,j}} - \frac{[\bA]_{k,k}}{[\bA]_{j,j}} \right)\frac{[\bA]_{j,j}}{[\bB^{(\tau-1)}]_{k,k} ([\bA]_{k,k} - [\bA]_{j,j})} + 1\right)^\inv - \alpha \right).\label{eq:B_sign}
\end{align}
\end{theorem}

If we consider a pair of diagonals of $\bD$, where $d_k > d_j$, then for $\alpha = 0$, the fraction $\frac{[\bB^{(\tau)}]_{k,k}}{[\bB^{(\tau)}]_{j,j}}$ is strictly increasing in $\tau$, due to the r.h.s. of \eqref{eq:B_diagonal_evolution} inside the parenthesis being strictly larger than $1$. Hence, the diagonals corresponding to smaller eigenvalues shrink faster than the larger ones as $\tau$ increases. However, for $\alpha \in (0,1)$ we can not ensure this behaviour, but at step $\tau$ we are able to predict the behaviour at step $\tau + 1$, by using \eqref{eq:B_sign}. Thus, when we include the ground-truth targets in our distillation procedure we do not consistently increase the regularization with each distillation step, but can potentially obtain a solution which does not sparsify any further. We now turn our attention to the question of how to pick the $\alpha^{(\tau)}$'s in an optimal manner, and find that it can be done if we relax the condition that the weights are restricted to the interval $[0,1]$.

\subsection{Closed Form Optimal Weighting Parameter}\label{sec:optimal_alpha}
Recall, $\tilde{\bX} \in \R^{m \times d}$ is the matrix of validation inputs and $\tilde{\by} \in \R^{m}$ the vector of validation targets. If we allow $\alpha^{(\tau)} \in \R$, we can find an \textit{optimal} $\alpha^{(\tau)}$ (which is a non-trivial function of $\lambda$) at each step $\tau$, denoted by $\alpha^{\star(\tau)}$.\footnote{If $\alpha^{\star(\tau)} \notin [0,1]$, the sign of either the first or second term of \eqref{eq:distill_objective} becomes negative, indicating either too strong or weak regularization of the previous distillation step, and one might fear this affects distillation performance. However, simply clipping of $\alpha^{\star(\tau)}$ to be in $[0,1]$ alleviates this, at the cost of requiring a larger $\tau$.} Here, \textit{optimal} denotes the value for which the validation MSE is minimized.

\begin{theorem}
Fix $\tau \geq 2$, $\lambda > 0$ and $\alpha^{(2)}, \dots, \alpha^{(\tau-1)} \in \R$, then
\begin{align}\label{eq:optimal_alpha}
    \alpha^{\star(\tau)} = \argmin_{\alpha^{(\tau)} \in \R} \norm{\tilde{\by} - f(\tilde{\bX}, \hat{\bbeta}^{(\tau)})}_2^2 = 1 - \frac{\left(\tilde{\by}^{(\tau)}_{\alpha=0} - \tilde{\by}^{(1)}\right)^\trans \left(\tilde{\by} - \tilde{\by}^{(1)}\right)}{\norm{\tilde{\by}^{(\tau)}_{\alpha=0} - \tilde{\by}^{(1)}}_2^2}
\end{align}
where $\tilde{\by}^{(1)} = f(\tilde{\bX}, \hat{\bbeta}^{(1)})$, and $\tilde{\by}^{(\tau)}_{\alpha=0} = f(\tilde{\bX}, \hat{\bbeta}^{(\tau)}_{\alpha=0})$.
\end{theorem}

Since neither $\tilde{\by}^{(1)}$ nor $\tilde{\by}^{(\tau)}_{\alpha=0}$ depend on the choice of $\alpha^{(\tau)}$, we can calculate $\alpha^{\star(\tau)}$ recursively as presented in Algorithm \ref{alg:optimal_alpha_procedure}, where $\alpha^{\star(\tau)}$ has the closed form in \eqref{eq:optimal_alpha}. In combination with the diagonalization results of Section \ref{sec:sparsification} we can efficiently calculate the solutions. This should be compared to performing grid-search for $\alpha$ with $g$ equidistant values on $[0,1]$ in order to approximate the optimal $\alpha$, which requires $g(\tau-1) + 1$ model fits if one uses the same $\alpha$ for each sequence of $\tau \geq 2$ steps ($g^{(\tau-1)}$ if $\alpha$ is not fixed across distillation steps). However, by Algorithm \ref{alg:optimal_alpha_procedure} it is sufficient to perform $2(\tau-1)+1$ model fits, and obtain the exact optimal value instead of an approximated value. In Section \ref{sec:optimal_alpha_DL} we apply Algorithm \ref{alg:optimal_alpha_procedure} to approximate $\alpha^{\star(\tau)}$ in a deep learning setting.

\begin{algorithm}[H]
        Calculate $\hat{\bbeta}^{(1)}$ from \eqref{eq:beta_minimizer} (with any $\alpha^{(1)}$)\;
        Calculate $\tilde{\by}^{(1)} = f(\tilde{\bX}, \hat{\bbeta}^{(1)})$\;
        \For{$t = 2$ \KwTo $\tau$}{
            \nosemic Calculate $\hat{\bbeta}^{(t)}_{\alpha=0}$ from \eqref{eq:beta_minimizer} and $\tilde{\by}^{(t)}_{\alpha=0} = f(\tilde{\bX}, \hat{\bbeta}^{(t)}_{\alpha=0})$\;
            \nosemic Solve: $\alpha^{\star(t)} = \argmin\limits_{\alpha \in \R} \norm{\tilde{\by} - \left(\alpha\tilde{\by}^{(1)} + (1-\alpha) \tilde{\by}^{(t)}_{\alpha=0} \right)}^2_2$\;
            \nosemic Calculate $\hat{\bbeta}^{(t)}$ from \eqref{eq:beta_minimizer} with $\alpha^{\star(t)}$\;
        }
    \caption{Calculate $\hat{\bbeta}^{(\tau)}$ and $\alpha^{\star(\tau)}$ for $\tau \geq 2$.}
    \label{alg:optimal_alpha_procedure}
\end{algorithm}

\subsection{Infinite Number of Self-Distillation Steps}\label{sec:infinite_steps}
We now prove that if we were to perform an infinite number of distillations steps ($\tau \to \infty$) with a fixed $\alpha$ (i.e. $\alpha^{(2)} = \dots = \alpha^{(\tau)} = \alpha$) the solution would solve the classical kernel ridge regression problem, with an amplified regularization parameter (by $\alpha^\inv$) if $\alpha > 0$.
Observe that, when $\alpha = 0$ and $\tau \to \infty$, \eqref{eq:recurrent_y} and \eqref{eq:recurrent_f} yield that the predictions $\by^{(\infty)}$ and $f(\bx, \hat{\bbeta}^{(\infty)})$ collapse to the zero-solution for any $\bx \in \R^p$ as expected from \citet{mobahi2020selfdistillation}.

\begin{theorem}\label{thm:kernel_limit}
Let $\by^{(\tau)}, \hat{\bbeta}^{(\tau)}$, and $f(\cdot, \hat{\bbeta}^{(\tau)})$ be defined as above, and $\alpha \in (0,1]$, then the following limits hold
\begin{align}
    \by^{(\infty)} &\eqdef \lim_{\tau \to \infty} \by^{(\tau)} = \bK\left(\bK + \frac{\lambda}{\alpha}\bI_n \right)^{-1} \by \label{eq:Y_limit}\\
    f(\bx, \hat{\bbeta}^{(\infty)}) &\eqdef  \lim_{\tau \to \infty} f(\bx, \hat{\bbeta}^{(\tau)})
    = \alpha f(\bx, \hat{\bbeta}^{(1)}) + (1-\alpha)f(\bx, \hat{\bgamma}^{(\infty)}) \nonumber
\end{align}
where \eqref{eq:Y_limit} corresponds to \emph{classical} kernel ridge regression with amplified regularization parameter $\frac{\lambda}{\alpha}$, and we let $\hat{\bgamma}^{(\infty)}$ denote the kernel ridge regression parameter associated with solving another kernel ridge regression on the targets $\by^{(\infty)}$ with regularization parameter $\lambda$. Furthermore, the convergence $\lim_{\tau \to \infty} \by^{(\tau)}$ is of linear rate.
\end{theorem}

If $\alpha > 0$, then by \eqref{eq:kernel_with_SVD} and Theorem \ref{thm:kernel_limit}, we have that $\by^{(\infty)} = \sum_{j=1}^p \bv_j \frac{d_j}{d_j + \frac{\lambda}{\alpha}} \bv_j^\trans \by$ and we shrink the eigenvectors with smallest eigenvalues, corresponding to the directions with least variance, the most. Furthermore, if $\alpha > 0$ the limiting solution is a non-zero kernel ridge regression with regularization parameter $\lambda/\alpha \geq \lambda$, causing the eigenvectors associated with the smallest eigenvalues to shrink even more than in the original solution.

Our results gives a theoretical explanation for why one should treat $\alpha^{(\tau)}$ as an adjustable hyperparameter to fine-tune the amount of regularization that self-distillation impose for a particular problem, and that it can be chosen in an optimal way for kernel ridge regression. In the following we provide an illustrative example, and in Section \ref{sec:optimal_alpha_DL} we estimate the optimal weighting parameter for deep learning using an adaptation of Algorithm \ref{alg:optimal_alpha_procedure}.

\subsection{Illustrative example}\label{sec:illustrative_example}
Consider the training dataset $\D$ where $\X = \{0, 0.1, \dots, 0.9, 1\}$ and $\Y = \{\sin(2\pi x) + \varepsilon \mid x \in \X\}$, and $\varepsilon$ is sampled from a zero-mean Gaussian random variable with standard deviation $0.5$. Let $\varphi$ be the Radial Basis Function kernel, i.e. $\kappa(\bx_i, \bx_j) = e^{-\gamma \norm{\bx_i - \bx_j}_2^2}$, where we choose $\gamma = \frac{1}{80}$, and let $\lambda = 0.2$ and consider the three cases; (\emph{a}) $\alpha = 0$, (\emph{b}) $\alpha = 0.25$, and (\emph{c}) step-wise optimal $\alpha^{\star(\tau)}$.

\begin{figure}[htbp]
    \centering
    \begin{subfigure}[b]{0.32\linewidth}
        \centering
        \includegraphics[width=\linewidth]{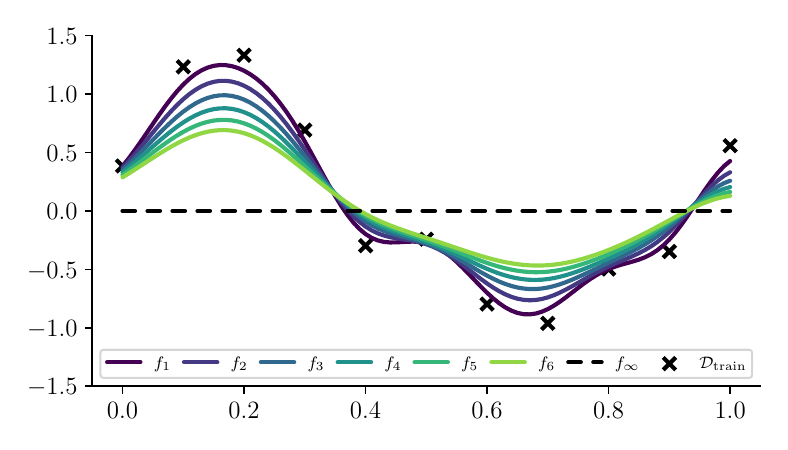}
        \caption{$\alpha = 0$}
        \label{fig:distill_no_GT}
    \end{subfigure}
    \hfill
    \begin{subfigure}[b]{0.32\linewidth}
        \centering
        \includegraphics[width=\linewidth]{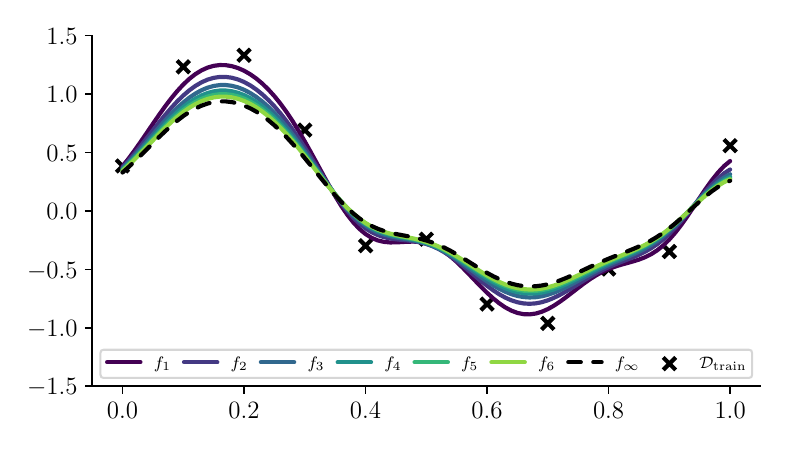}
        \caption{$\alpha = 0.25$}
        \label{fig:distill_GT}
    \end{subfigure}
    \hfill
    \begin{subfigure}[b]{0.32\linewidth}
        \centering
        \includegraphics[width=\linewidth]{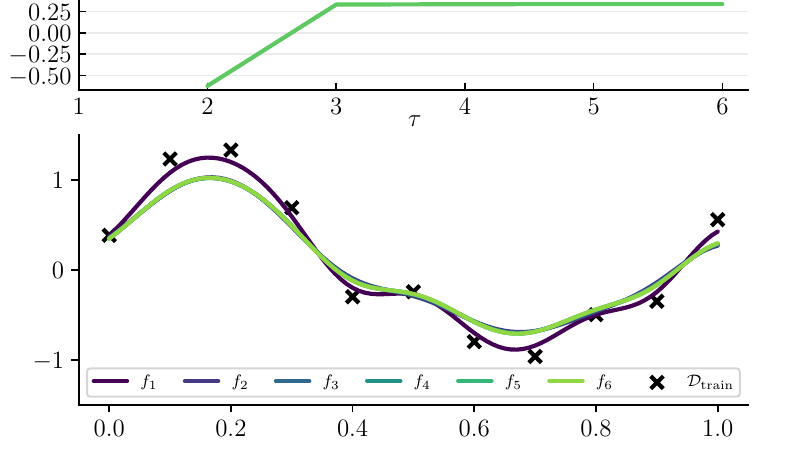}
        \caption{$\alpha^{\star(\tau)}$}
        \label{fig:distill_optimal}
    \end{subfigure}
    \caption{Six steps of self-distillation with (\subref{fig:distill_no_GT}) zero limiting solution (dashed), (\subref{fig:distill_GT}) non-zero limiting solution (dashed), and (\subref{fig:distill_optimal}) optimal step-wise $\alpha^{\star(\tau)}$. Training examples are represented with $\times$.}
    \label{fig:illustrative_distill}
\end{figure}

As illustrated in Figure \ref{fig:distill_no_GT} for case (\emph{a}), the regularization imposed by self-distillation initially improves the quality of the solution, but eventually overregularize and the solutions underfit the data, and will eventually converge to the zero-solution. Using $\alpha > 0$ (see Figure \ref{fig:distill_GT}), and more specifically $\alpha = 0.25$, reduce the imposed regularization and increases the stability of the distillation procedure; i.e. the solutions differ much less between each distillation step. This allows for a more dense exploration of solutions during iterated distillation steps, where increasing $\alpha$ reduces the difference between solutions from two consecutive steps, but also reduces the space of possible solutions as the limit, $f(\cdot, \hat{\bbeta}^{(\infty)})$, approaches the initial solution $f(\cdot, \hat{\bbeta}^{(1)})$ quickly.\footnote{As expected by Theorem \ref{thm:kernel_limit}, we experience a fast convergence to the limit; usually less than 10 iterations are sufficient to converge}. However, choosing the step-wise optimal $\alpha^{\star(\tau)}$ yields minuscule changes to the solution for $\tau > 2$, and a single step of distillation is effectively enough. Furthermore, for $\tau \geq 3$, all $\alpha^{\star(\tau)}$ are approximately equal, and the distillation procedure has reached an equilibrium.\footnote{If we clip $\alpha^{\star(\tau)}$ to be in $[0,1]$, the $\alpha^{\star(\tau)}$ converges at $\tau = 4$ rather than $\tau = 3$.}

As expected from Lemma \ref{lem:recurrent_B} and Theorem \ref{thm:B_sparsification}, Figure \ref{fig:illustrative_B} verifies that both in case (\emph{a}) and (\emph{b}), the diagonal of $\bB^{(\tau)}$ is decreasing in $\tau$ and the diagonal coordinates corresponding to smaller eigenvalues shrink faster than those corresponding to larger eigenvalues. Without loss of generality we can assume $d_1 < d_2 < \dots < d_n$, and for $k = 1, \dots n-1$ and any $\tau \geq 1$ define $R^{(\tau)}_k \eqdef \smash{ \frac{[\bB^{(\tau)}]_{k+1, k+1}}{[\bB^{(\tau)}]_{k, k}}}$. We expect $R^{(\tau)}_k$ to be strictly increasing in $\tau$ for all $k$ in case (\emph{a}), but for case (\emph{b}) we can make no such guarantee. Both of these properties are verified in Figure \ref{fig:illustrative_Rk}.

\begin{figure}[htbp]
    \centering
    \begin{subfigure}[b]{0.32\linewidth}
        \centering
        \includegraphics[width=\linewidth]{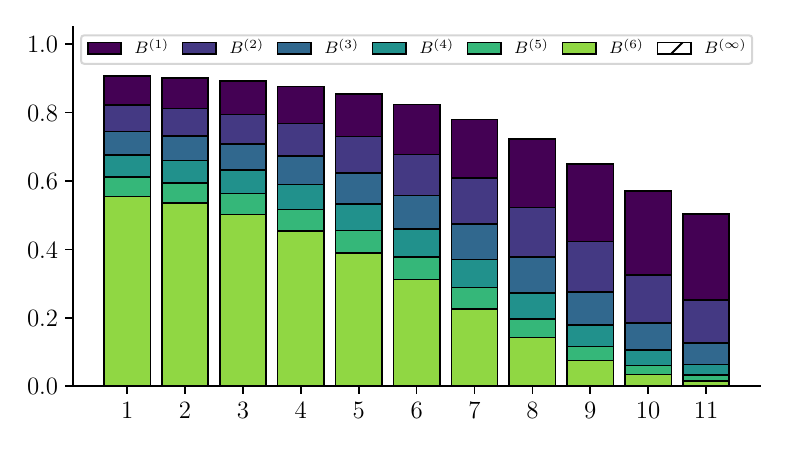}
        \caption{$\alpha = 0$}
        \label{fig:B_no_GT}
    \end{subfigure}
    \hfill
    \begin{subfigure}[b]{0.32\linewidth}
        \centering
        \includegraphics[width=\linewidth]{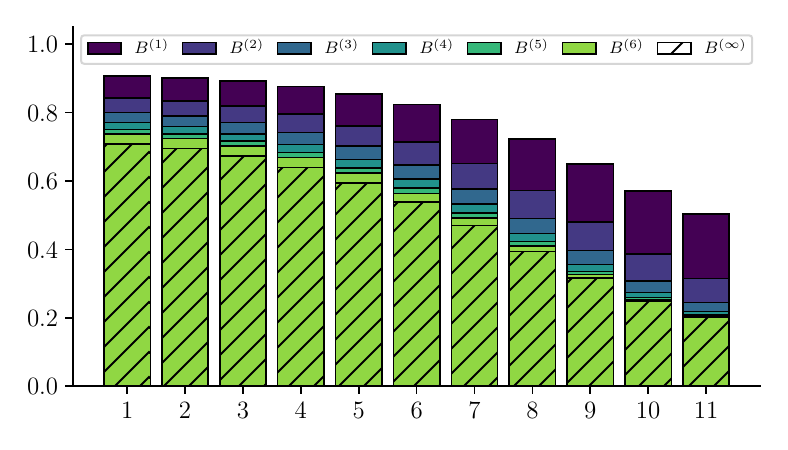}
        \caption{$\alpha = 0.25$}
        \label{fig:B_GT}
    \end{subfigure}
    \hfill
    \begin{subfigure}[b]{0.32\linewidth}
        \centering
        \includegraphics[width=\linewidth]{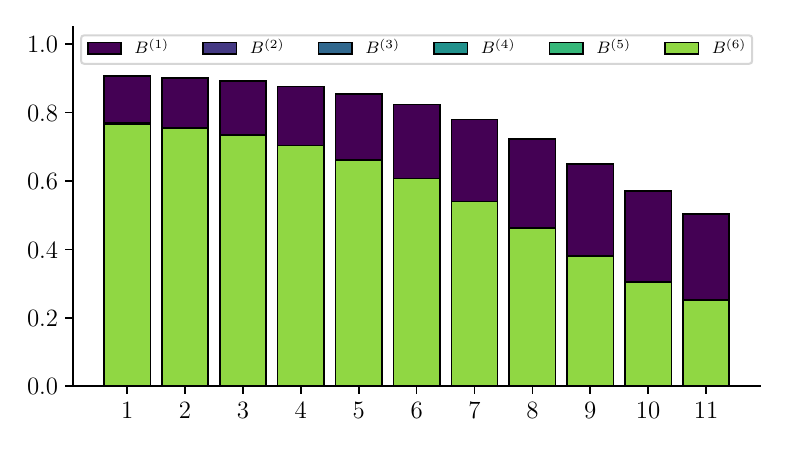}
        \caption{$\alpha^{\star(\tau)}$}
        \label{fig:B_optimal}
    \end{subfigure}
    \caption{Diagonal of $\bB^{(\tau)}$ for $\tau = 1, \dots, 6$ associated with Figure \ref{fig:illustrative_distill}. Note, the plots are overlaid, but since the diagonal of $\bB^{(\tau)}$ decrease in $\tau$, all values until convergence are visible. In (\subref{fig:B_no_GT}) we expect and observe strictly decreasing values in $\tau$ for all indices, until collapsing at $0$, but in (\subref{fig:B_GT}) and (\subref{fig:B_optimal}) the values converge to a non-zero limit.}
    \label{fig:illustrative_B}
\end{figure}

Finally, we observe that in case (\emph{a}), the values of $\bB^{(\tau)}$ shrink much faster than in case (\emph{b}), and eventually collapse to all zeros, whereas the latter is nearly converged after six iterations. Furthermore, case (\emph{a}) appear to obtain a more sparsified solution, as the smallest coordinates effectively diminishes, which is not true for case (\emph{b}). Furthermore, when directly comparing solutions from both cases with similar quality of fit, the solutions obtained with $\alpha = 0$ usually has smaller coordinates in $\bB^{(\tau)}$ than those obtained with larger values of $\alpha$.

\begin{figure}[htbp]
    \centering
    \begin{subfigure}[b]{0.32\linewidth}
        \centering
        \includegraphics[width=\linewidth]{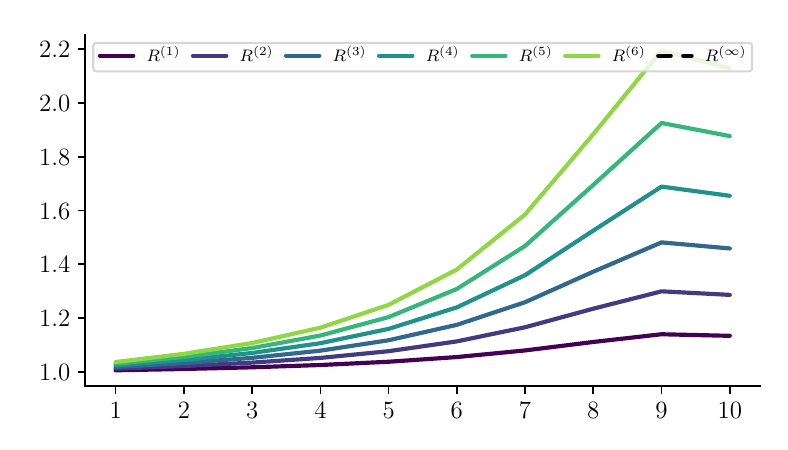}
        \caption{$\alpha = 0$}
        \label{fig:Rk_no_GT}
    \end{subfigure}
    \hfill
    \begin{subfigure}[b]{0.32\linewidth}
        \centering
        \includegraphics[width=\linewidth]{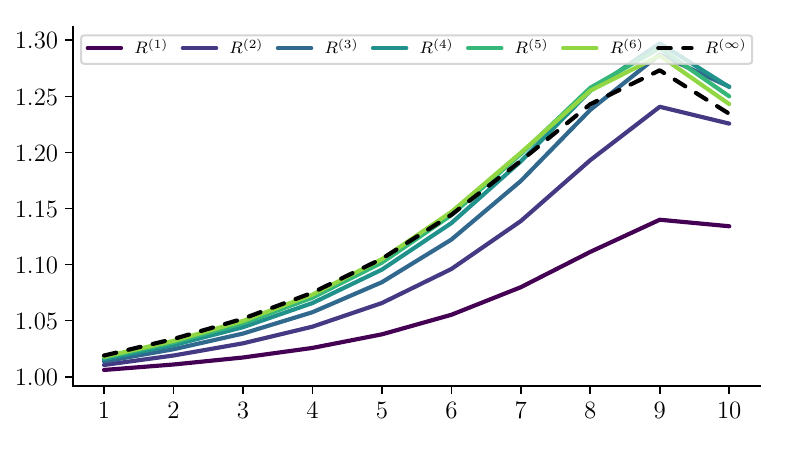}
        \caption{$\alpha = 0.25$}
        \label{fig:Rk_GT}
    \end{subfigure}
    \hfill
    \begin{subfigure}[b]{0.32\linewidth}
        \centering
        \includegraphics[width=\linewidth]{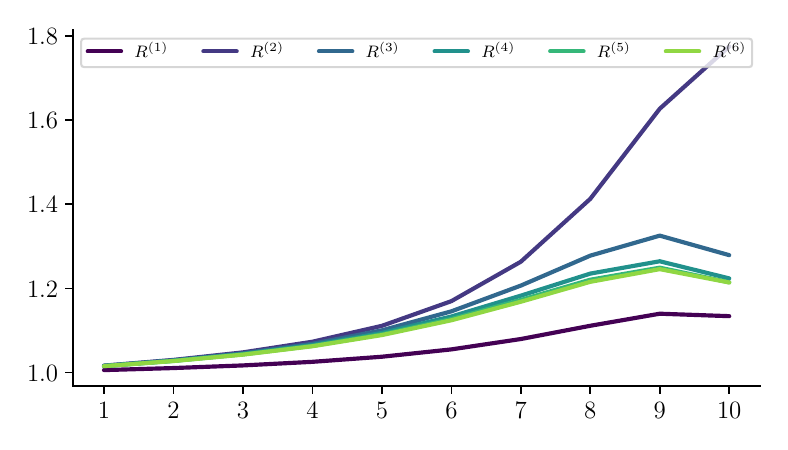}
        \caption{$\alpha^{\star(\tau)}$}
        \label{fig:Rk_optimal}
    \end{subfigure}
    \caption{Ratios, $R^{(\tau)}_k$ of the ordered diagonal of $\bB^{(\tau)}$ for all $\tau$. In (\subref{fig:Rk_no_GT}) we expect and observe strictly increasing values in $\tau$ for all $k$, but have no such guarantee in (\subref{fig:Rk_GT}) or (\subref{fig:Rk_optimal}). The x-axis corresponds to indices $k = 1, \dots, n-1$.}
    \label{fig:illustrative_Rk}
\end{figure}

\section{Approximate Optimal Weighting Parameter for Deep Learning}\label{sec:optimal_alpha_DL}

The following experiment aim at empirically evaluating the theoretical analysis above in a simple deep learning setting. In \eqref{eq:optimal_alpha} we find $\alpha^{\star(\tau)}$ on closed form when $f(\cdot, \hat{\bbeta}^{(\tau)})$ is a (self-distilled) kernel ridge regression. No closed form solution can be found for neural networks, but recent results show that (very) wide neural networks can be seen as kernel ridge regression solutions with the neural tangent kernel \citep{Jacot2018NeuralNetworks, Arora2019OnNet, Lee2019WideDescent, Lee2020GeneralizedNetworks}.

Thus, inspired by \eqref{eq:recurrent_f} we propose to estimate $\alpha^{\star(t)}$ for $t = 2,\dots,\tau$, denoted by $\hat{\alpha}^{(t)}$, for a neural network trained with self-distillation using an adapted Algorithm \ref{alg:optimal_alpha_procedure}. Let $\nn(\cdot, \btheta) \in \R^p$ be a neural network with vector of weights $\btheta$, and recursively for $\tau \geq 1$ let $\hat{\btheta}^{(\tau)}$ be the weights solving
\begin{align}\label{eq:DL_objective}
    \argmin_{\theta} \frac{\alpha^{(\tau)}}{2}\norm{\nn(\bX, \btheta) - \bY^{(1)}}_F^2 + \frac{1 - \alpha^{(\tau)}}{2}\norm{\nn(\bX, \btheta) - \bY^{(\tau-1)}}_F^2 + \frac{\lambda}{2}\norm{\btheta}_2^2,
\end{align}
with $\alpha^{(\tau)} = \hat{\alpha}^{(\tau)}$ and where $\bY^{(\tau)} \in \R^{n \times p}$.\footnote{We treat class labels as $p$-dimensional one-hot encoded vectors and use norm of the difference between the predicted class probabilities and the one-hot vectors.} Furthermore, let $\hat{\btheta}^{(\tau)}_{\alpha=0}$ be the weights associated with minimizing \eqref{eq:DL_objective} with $\alpha^{(\tau)} = 0$, and $\tilde{\bY}^{(\tau)}_{\alpha=0} \eqdef \nn(\tilde{\bX}, \hat{\btheta}^{(\tau)}_{\alpha=0})$ as well as $\tilde{\bY}^{(\tau)} \eqdef \nn(\tilde{\bX}, \hat{\btheta}^{(\tau)})$ be the predictions on the validation input $\tilde{\bX}$. Then, following Algorithm \ref{alg:optimal_alpha_procedure} with $\norm{\cdot}_2$ replaced by $\norm{\cdot}_F$, and \eqref{eq:DL_objective} rather than \eqref{eq:distill_objective} we can calculate the estimates $\hat{\alpha}^{(t)}$.
These estimates yield comparable predictive performance to the best fixed $\alpha^{(\tau)}$ (found with time-consuming grid search), but only require one additional model fit per distillation step; i.e. $2(\tau-1)+1$ fits compared to $g(\tau-1) +1$ for a grid search over $g$ values. See Figure \ref{fig:optimal_resnet50_by_step} for results and supplementary material for experimental details.

\subsection{Experiment}\label{sec:DL_experiments}
We perform self-distillation with ResNet-50 \citep{He2016DeepRecognition} networks on CIFAR-10 \citep{Krizhevsky2009LearningImages}, with minor pre-processing and augmentations. The model is initialized randomly at each step\footnote{Note, we initialize the models equally across all $\alpha$ for one experiment, but alter the seed for initialization between experiments.} and trained according to the above with either estimated optimal parameters, $\hat{\alpha}^{(\tau)}$, or fixed $\alpha$ for all steps. We use the network weights from the last iteration of training at each distillation step for the next step, irrespective of whether a better model occurred earlier in the training. Our models are trained for a fixed $75$ epochs and each experiment is repeated with $4$ different random seeds over $11$ chains of distillation steps, corresponding to $\alpha \in \{0.0, 0.1, \dots, 0.9\}$ and $\hat{\alpha}^{(\tau)}$, with the first model initialized identically across all chains. The accuracy reported at the $\tau$'th step is based on comparing the training and validation predictions, $\bY^{(\tau)}$ and $f(\tilde{\bX}, \hat{\bbeta}^{(\tau)})$ with the original training and validation targets; $\bY$ and $\tilde{\bY}$.

\begin{figure}[htbp]
    \centering
    \begin{subfigure}[b]{0.44\linewidth}
        \centering
        \includegraphics[width=\linewidth]{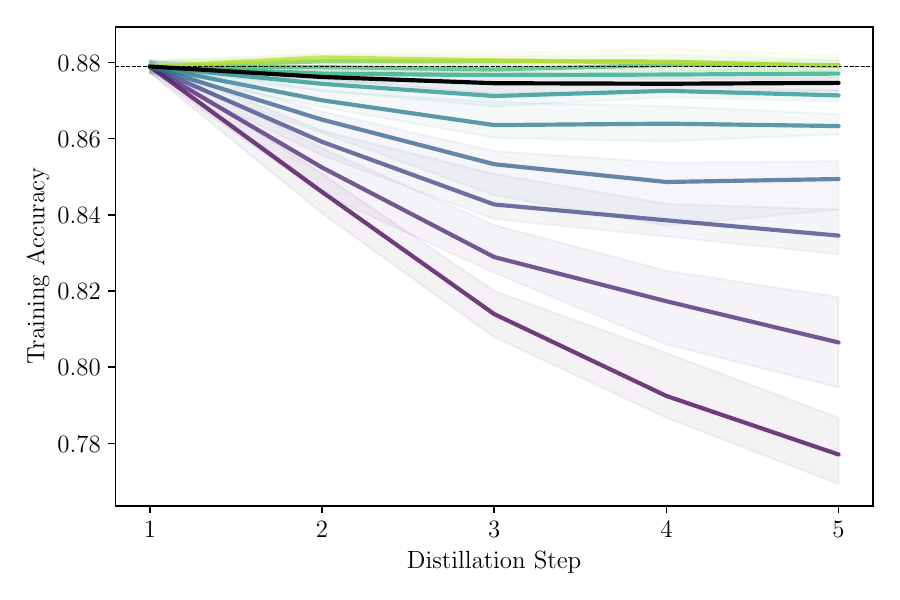}
    \end{subfigure}
    \hfill
    \begin{subfigure}[b]{0.515\linewidth}
        \centering
        \includegraphics[width=\linewidth]{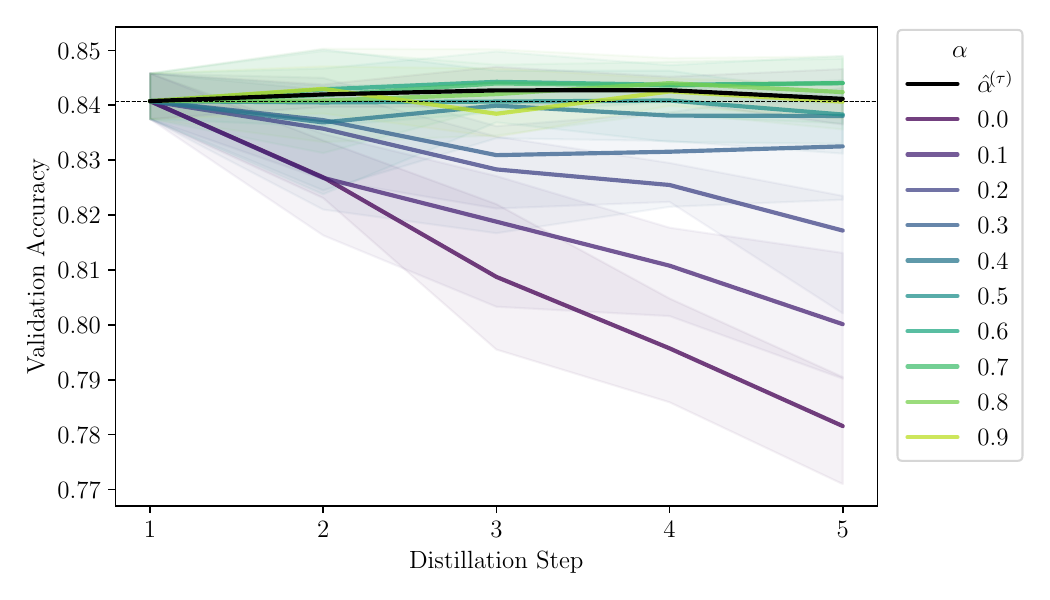}
    \end{subfigure}
    \caption{Training and validation accuracy for five distillation steps with ResNet-50 models on CIFAR-10. Comparing fixed $\alpha^{(t)}$ for $t=2\dots,\tau$ and estimating optimal weight with $\hat{\alpha}^{(t)}$ at each step. The experiment is repeated four times and the mean (and max/min in shaded) is reported.}
    \label{fig:optimal_resnet50_by_step}
\end{figure}

\section{Conclusion}
In this paper, we provided theoretical arguments for the importance of weighting the teacher outputs with the ground-truth targets when performing self-distillation with kernel ridge regressions along with a closed form solution for the optimal weighting parameter. We proved how the solution at any (possibly infinite) distillation step can be calculated directly from the initial distillation step, and that self-distillation for an infinite number of steps corresponds to a classical kernel ridge regression solution with amplified regularization parameter. We showed both empirically and theoretically that the weighting parameter $\alpha$ determines the amount of regularization imposed by self-distillation, and empirically supported our results in a simple deep learning setting.

\subsection{Future Research Directions}
Interesting directions of future research are on rigorously connecting neural networks and kernel methods in a knowledge distillation setting, extend to other objective functions than MSE as well as including intermediate model statistics in the distillation procedure. Finally, a larger empirical study of the connection between the choice of $\alpha$ and the degree of overfitting is interesting as well.

\subsection*{Acknowledgement}
We would like to thank GenomeDK and Aarhus University for providing computational resources that contributed to these research results. Furthermore, we would like to thank Daniel Borup and Ragnhild Ø. Laursen for comments and discussion, as well as Google Researcher Hossein Mobahi and Mehrdad Farajtabar (Deepmind) for clarifications on their experimental setup. We also thank the anonymous reviewers of the NeurIPS 2021 conference for their comments. Kenneth Borup is partly financed by Aarhus University Centre for Digitalisation, Big Data and Data Analytics (DIGIT).

\bibliography{paper}
\bibliographystyle{apalike}

\clearpage
\appendix
\include{proofs}
\include{experiments}
\include{ntk}
\include{constrained}

\end{document}

%% file: preamble.tex
\usepackage[utf8]{inputenc} 
\usepackage[T1]{fontenc}    
\usepackage{hyperref}       
\usepackage{url}            
\usepackage{booktabs}       
\usepackage{amsfonts}       
\usepackage{nicefrac}       
\usepackage{microtype}      
\usepackage{xcolor}         

\usepackage{graphicx}           
\usepackage{natbib}             
\usepackage{amsmath}            
\usepackage{amssymb}            
\usepackage{amsthm}             
\usepackage{enumitem}           
\usepackage{parskip}            
\usepackage{bm}                 
\usepackage{dsfont}             
\usepackage{caption}            
\usepackage{subcaption}         
\usepackage[ruled]{algorithm2e} 
\newcommand{\nosemic}{\SetEndCharOfAlgoLine{\relax}}
\let\oldnl\nl
\newcommand{\nonl}{\renewcommand{\nl}{\let\nl\oldnl}}
\usepackage{wrapfig}

\newcommand{\D}{\mathcal{D}}

\newcommand{\F}{\mathcal{F}}
\renewcommand{\L}{\mathcal{L}}
\newcommand{\R}{\mathbb{R}}
\newcommand{\X}{\mathcal{X}}
\newcommand{\Y}{\mathcal{Y}}

\newcommand{\norm}[1]{\left\lVert #1 \right\rVert}
\newcommand{\argmin}{\mathop{\mathrm{argmin}}}

\newcommand{\bbeta}{\bm{\beta}}
\newcommand{\bgamma}{\bm{\gamma}}
\newcommand{\btheta}{\bm{\theta}}

\newcommand\eqdef{\mathrel{\overset{\makebox[0pt]{\mbox{\normalfont\tiny\sffamily def}}}{=}}} 
\newcommand{\trans}{{\intercal}}
\newcommand{\inv}{{-1}}
\newcommand{\nn}{f_{\mathrm{nn}}}

\newtheorem{theorem}{Theorem}[section]
\newtheorem{lemma}[theorem]{Lemma}
\newtheorem{proposition}[theorem]{Proposition}

\newcommand\bA{\mathbf{A}}
\newcommand\bB{\mathbf{B}}
\newcommand\bC{\mathbf{C}}
\newcommand\bD{\mathbf{D}}
\newcommand\bE{\mathbf{E}}

\newcommand\bG{\mathbf{G}}

\newcommand\bI{\mathbf{I}}

\newcommand\bK{\mathbf{K}}

\newcommand\bV{\mathbf{V}}

\newcommand\bX{\mathbf{X}}
\newcommand\bY{\mathbf{Y}}

\newcommand\ba{\mathbf{a}}

\newcommand\bg{\mathbf{g}}

\newcommand\bp{\mathbf{p}}

\newcommand\bx{\mathbf{x}}
\newcommand\bv{\mathbf{v}}

\newcommand\by{\mathbf{y}}
\newcommand\bz{\mathbf{z}}

%% file: proofs.tex
\section{Proofs}\label{sec:proofs}
This section includes all proofs referenced in the main part of the paper, along with the associated theorems and lemmas for completeness.

\begin{theorem}
Let $\by^{(\tau)}, \hat{\bbeta}^{(\tau)}$, and $f(\cdot, \hat{\bbeta}^{(\tau)})$ be defined as above. Fix $\alpha^{(2)}, \dots, \alpha^{(\tau)} \in [0, 1)$, and let $\eta(i, \tau) \eqdef \prod_{j=i}^{\tau} \left(1-\alpha^{(j)} \right)$, then for $\tau \geq 1$, we have that
\begin{align*}
    &\by^{(\tau)} = \left( \sum_{i=2}^{\tau} \alpha^{(i)} \eta(i+1, \tau) \left(\bK\left(\bK + \lambda \bI_n\right)^{-1}\right)^{\tau-i+1} + \eta(2, \tau)\left(\bK\left(\bK + \lambda \bI_n\right)^{-1}\right)^{\tau} \right) \by,\\
    &f(\bx, \hat{\bbeta}^{(\tau)}) = \alpha^{(\tau)} f(\bx, \hat{\bbeta}^{(1)}) + (1-\alpha^{(\tau)})f(\bx, \hat{\bbeta}^{(\tau)}_{\alpha=0})
\end{align*}
for any $\bx \in \R^{d}$, where $\hat{\bbeta}^{(\tau)}_{\alpha=0}$ is the minimizer in \eqref{eq:beta_minimizer} with $\alpha^{(\tau)} = 0$.
\end{theorem}
\begin{proof}
We prove the theorem by induction, where we let $\tilde{\bK} \eqdef \bK(\bK + \lambda\bI_n)^\inv$. For $\tau = 1$, the result hold trivially, and thus, assume it hold for $\tau = t$. Since $\bbeta^{(t+1)} = \varphi(\bX)^\trans(\bK + \lambda\bI_n)^\inv\left(\alpha^{(t+1)}\by + (1-\alpha^{(t+1)})\by^{(t)}\right)$ we have that
\begin{align*}
    \by^{(t+1)} &= \varphi(\bX)\varphi(\bX)^\trans(\bK + \lambda\bI_n)^\inv\left(\alpha^{(t+1)}\by + (1-\alpha^{(t+1)})\by^{(t)}\right) \\
    &= \alpha^{(t+1)} \tilde{\bK} \by + (1-\alpha^{(t+1)})\tilde{\bK} \left( \sum_{i=2}^{t} \alpha^{(i)} \eta(i+1, t) \tilde{\bK}^{t-i+1} + \eta(2, t)\tilde{\bK}^{t} \right) \by \\
    &= \alpha^{(t+1)} \tilde{\bK} \by + \left( \sum_{i=2}^{t} \alpha^{(i)} \eta(i+1, t+1) \tilde{\bK}^{(t+1)-i+1} + \eta(2, t+1)\tilde{\bK}^{t+1} \right) \by \\
    &= \left( \sum_{i=2}^{t+1} \alpha^{(i)} \eta(i+1, t+1) \tilde{\bK}^{(t+1)-i+1} + \eta(2, t+1)\tilde{\bK}^{t+1} \right) \by
\end{align*}
which finalizes our induction proof for the first part. For the second part, note that it also holds trivially for $\tau = 1$. Thus assume, it holds for $\tau = t$, then by direct manipulations
\begin{align*}
    f(\bx, \bbeta^{(t+1)}) &= \kappa(\bx, \bX)^\trans(\bK + \lambda\bI_n)^\inv\left(\alpha^{(t+1)}\by + (1-\alpha^{(t+1)})\by^{(t)} \right) \\
    &= \alpha^{(t+1)} f(\bx, \bbeta^{(1)}) + (1-\alpha^{(t+1)})\kappa(\bx, \bX)^\trans(\bK + \lambda\bI_n)^\inv \by^{(t)} \\
    &= \alpha^{(t+1)} f(\bx, \bbeta^{(1)}) + (1-\alpha^{(t+1)})f(\bx, \hat{\bbeta}^{(t+1)}_{\alpha=0}),
\end{align*}
where we let $\hat{\bbeta}^{(t+1)}_{\alpha=0}$ denote the minimizer \eqref{eq:beta_minimizer} with $\alpha^{(t+1)} = 0$; i.e. minimizing the classical kernel ridge regression problem with targets $\by^{(t)}$.
\end{proof}

\begin{lemma}
Let $\bB^{(\tau)}$, and $\bA$ be defined as above, and let $\bB^{(0)} \eqdef \bI$. Then we can express $\bB^{(\tau)}$ recursively as
\begin{align*}
    \bB^{(\tau)} = \bA\left((1-\alpha^{(\tau)})\bB^{(\tau-1)} + \alpha^{(\tau)} \bI_n\right),
\end{align*}
and $[\bB^{(\tau)}]_{k,k} \in [0,1]$ is (strictly) decreasing in $\tau$ for all $k \in [n]$ and $\tau \geq 1$ if $\alpha^{(2)} = \dots = \alpha^{(\tau)} = \alpha$.
\end{lemma}
\begin{proof}
The case, $\tau = 1$, is easy to verify, and we assume the claim holds for $\tau = t$. Then note that
\begin{align*}
    \bA\left((1-\alpha^{(t+1)})\bB^{(t)} + \alpha^{(t+1)} \bI_n \right) &= \sum_{i=2}^{t} \alpha^{(i)} \eta(i+1, t+1) \bA^{(t+1)-i+1} + \eta(2, t+1)\bA^{t+1} + \alpha^{(t+1)}\bA \\
    &= \sum_{i=2}^{t+1} \alpha^{(i)} \eta(i+1, t+1) \bA^{(t+1)-i+1} + \eta(2, t+1)\bA^{t+1} \\
    &= \bB^{(t+1)},
\end{align*}
finalizing the induction proof. Now, assume $\alpha^{(2)} = \dots = \alpha^{(\tau)} = \alpha$ and note that for any $k$ and $\tau \geq 1$, then
\begin{align*}
    [\bA]_k \left((1-\alpha)[\bB^{(\tau-1)}]_{k,k} + \alpha\right) = [\bB^{(\tau)}]_{k,k} \leq [\bB^{(\tau-1)}]_{k,k} = [\bA]_k \left((1-\alpha)[\bB^{(\tau-2)}]_{k,k} + \alpha\right),
\end{align*}
if and only if $[\bB^{(\tau-1)}]_{k,k} \leq [\bB^{(\tau-2)}]_{k,k}$, and iteratively, if and only if $[\bB^{(1)}]_{k,k} \leq [\bB^{(0)}]_{k,k}$. The latter is indeed true, since $\bB^{(1)} = \bA$, and finally, $\bA = \bI_n$ if and only if $\lambda = 0$.
\end{proof}

\begin{theorem}
Assume $\alpha^{(2)} = \dots = \alpha^{(\tau)} = \alpha$. Then, for any pair of diagonals of $\bD$, i.e. $d_k$ and $d_j$, where $d_k > d_j$, we have that for all $\tau \geq 1$,
\begin{align*}
    \frac{[\bB^{(\tau)}]_{k,k}}{[\bB^{(\tau)}]_{j,j}} &= \begin{cases} \frac{1 + \frac{\lambda}{d_j}}{1 + \frac{\lambda}{d_k}}, &\text{for } \alpha = 1, \\ \left(\frac{1 + \frac{\lambda}{d_j}}{1 + \frac{\lambda}{d_k}}\right)^\tau, &\text{for } \alpha = 0, \end{cases}
\end{align*}
and if we let $\mathrm{sgn}(\cdot)$ denote the sign function, i.e.
\begin{align*}
    \mathrm{sgn}(x) \eqdef \begin{cases} \;1 &\text{if} \quad x > 0 \\ \;0 &\text{if} \quad x = 0 \\ -1 &\text{if} \quad x < 0\end{cases},
\end{align*}
then for $\alpha \in (0, 1)$ we have that
\begin{align*}
    &\mathrm{sgn}\left(\frac{[\bB^{(\tau)}]_{k,k}}{[\bB^{(\tau)}]_{j,j}} - \frac{[\bB^{(\tau-1)}]_{k,k}}{[\bB^{(\tau-1)}]_{j,j}} \right) \nonumber\\
    &\quad= \mathrm{sgn}\left(\left( \left(\frac{[\bB^{(\tau-1)}]_{k,k}}{[\bB^{(\tau-1)}]_{j,j}} - \frac{[\bA]_{k,k}}{[\bA]_{j,j}} \right)\frac{[\bA]_{j,j}}{[\bB^{(\tau-1)}]_{k,k} ([\bA]_{k,k} - [\bA]_{j,j})} + 1\right)^\inv - \alpha \right).
\end{align*}
\end{theorem}
\begin{proof}
First note that
\begin{align*}
    \frac{[\bA]_{k,k}}{[\bA]_{j,j}} = \frac{\frac{d_k}{d_k + \lambda}}{\frac{d_j}{d_j + \lambda}} = \frac{1 + \frac{\lambda}{d_j}}{1 + \frac{\lambda}{d_k}},
\end{align*}
and for $\alpha = 1$, \eqref{eq:recurrent_B} amounts to $\bB^{(\tau)} = \bA$, which gives the first result. For $\alpha = 0$, \eqref{eq:recurrent_B} amounts to $\bB^{(\tau)} = \bA^{\tau}$, and the second result follows. For the remainder we denote $[\bB^{(\tau-1)}]_{k,k}$ by $\bB_k$ and $[\bA]_{k,k}$ by $\bA_k$ to simplify notation. We investigate the case where both r.h.s. and l.h.s. equals zero. Thus, for $\alpha \in (0, 1)$, we observe that if
\begin{align*}
    \frac{\bB_k}{\bB_j} &= \frac{\bA_k}{\bA_j} \frac{(1-\alpha)\bB_k + \alpha}{(1-\alpha)\bB_j + \alpha} = \frac{\bA_k}{\bA_j}\frac{\frac{1-\alpha}{\alpha} \bB_k + 1}{\frac{1-\alpha}{\alpha}\bB_j + 1},
\end{align*}
then we have that
\begin{align*}
    \bB_j &= \frac{1}{\frac{\bA_k}{\bA_j}\left(\frac{\frac{1-\alpha}{\alpha}\bB_k + 1}{\bB_k}\right) - \frac{1-\alpha}{\alpha}} = \frac{\bB_k}{\frac{\bA_k}{\bA_j}\left(\frac{1-\alpha}{\alpha}\bB_k + 1\right) - \frac{1-\alpha}{\alpha}\bB_k} \\
    \frac{1-\alpha}{\alpha}\bB_j + 1 &= \frac{\frac{1-\alpha}{\alpha}\bB_k}{\frac{\bA_k}{\bA_j}\left(\frac{1-\alpha}{\alpha}\bB_k + 1\right) - \frac{1-\alpha}{\alpha}\bB_k} + 1 = \frac{\frac{\bA_k}{\bA_j}\left(\frac{1-\alpha}{\alpha}\bB_k + 1\right)}{\frac{\bA_k}{\bA_j}\left(\frac{1-\alpha}{\alpha}\bB_k + 1\right) - \frac{1-\alpha}{\alpha}\bB_k},
\end{align*}
which in turn yield that
\begin{align*}
    \frac{\bB_k}{\bB_j} &= \frac{\bA_k}{\bA_j} \left(\frac{1-\alpha}{\alpha}\bB_k + 1\right) \left(\frac{\frac{\bA_k}{\bA_j}\left(\frac{1-\alpha}{\alpha}\bB_k + 1\right) - \frac{1-\alpha}{\alpha}\bB_k}{\frac{\bA_k}{\bA_j}\left(\frac{1-\alpha}{\alpha}\bB_k + 1\right)} \right) \\
    &= \frac{\bA_k}{\bA_j}\left(\frac{1-\alpha}{\alpha}\bB_k + 1 \right) - \frac{1 - \alpha}{\alpha}\bB_k.
\end{align*}
Now, observe that $0 = \alpha - \left( \left(\frac{\bB_k}{\bB_j} - \frac{\bA_k}{\bA_j} \right) \frac{\bA_j}{\bB_k (\bA_k - \bA_j)} + 1\right)^\inv$ yield that
\begin{align*}
    \frac{\bB_k}{\bB_j} &= \frac{1 - \alpha}{\alpha} \frac{\bB_k(\bA_k- \bA_j)}{\bA_j} + \frac{\bA_k}{\bA_j} = \frac{\bA_k}{\bA_j}\left(\frac{1-\alpha}{\alpha}\bB_k + 1 \right) - \frac{1 - \alpha}{\alpha}\bB_k.
\end{align*}
Thus, similar calculations with $>$ and $<$ instead of $=$, completes the claim.
\end{proof}

\begin{theorem}
Fix $\tau \geq 2$, $\lambda > 0$ and $\alpha^{(2)}, \dots, \alpha^{(\tau-1)} \in \R$, then
\begin{align*}
    \alpha^{\star(\tau)} &= \argmin_{\alpha^{(\tau)} \in \R} \norm{\tilde{\by} - f(\tilde{\bX}, \hat{\bbeta}^{(\tau)})}_2^2 \\
    &= \frac{\left(\frac{\partial}{\partial \alpha^{(\tau)}} f(\tilde{\bX}, \hat{\bbeta}^{(\tau)})\right)^\trans \left(\tilde{\by} - \tilde{\by}^{(1)}\right)}{\norm{\frac{\partial}{\partial \alpha^{(\tau)}} f(\tilde{\bX}, \hat{\bbeta}^{(\tau)})}^2} + 1, \\
    &= 1 - \frac{\left(\tilde{\by}^{(\tau)}_{\alpha=0} - \tilde{\by}^{(1)}\right)^\trans \left(\tilde{\by} - \tilde{\by}^{(1)}\right)}{\norm{\tilde{\by}^{(\tau)}_{\alpha=0} - \tilde{\by}^{(1)}}_2^2}
\end{align*}
where $\tilde{\by}^{(1)} = f(\tilde{\bX}, \hat{\bbeta}^{(1)})$, and $\tilde{\by}^{(\tau)}_{\alpha=0} = f(\tilde{\bX}, \hat{\bbeta}^{(\tau)}_{\alpha=0})$.
\end{theorem}
\begin{proof}
Let $\L(\alpha^{(\tau)}, \lambda) = \norm{\tilde{\by} - f(\tilde{\bX}, \hat{\bbeta}^{(\tau)})}^2$, where $f$ depends on $\alpha^{(\tau)}$ and $\lambda$ through $\hat{\bbeta}^{(\tau)}$. Note that,
\begin{align*}
    f(\tilde{\bX}, \hat{\bbeta}^{(\tau)}) &= \kappa(\tilde{\bX}, \bX)(\bK +\lambda\bI)^{-1}\left( \alpha^{(\tau)}\by + (1-\alpha^{(\tau)})\by^{(\tau-1)} \right) \\
    \frac{\partial}{\partial \alpha^{(\tau)}} f(\tilde{\bX}, \hat{\bbeta}^{(\tau)}) &= \kappa(\tilde{\bX}, \bX)(\bK +\lambda\bI)^{-1}\left(\by - \by^{(\tau-1)} \right).
\end{align*}
Then for fixed $\lambda > 0$, we have that
\begin{align*}
    \frac{\partial}{\partial \alpha^{(\tau)}}\L(\alpha, \lambda) &= \left(\frac{\partial}{\partial \alpha^{(\tau)}}f(\tilde{\bX}, \hat{\bbeta}^{(\tau)}) \right)^\trans \left( 2 f(\tilde{\bX}, \hat{\bbeta}^{(\tau)}) - 2\tilde{\by} \right)
\end{align*}
and since we can decompose $f(\tilde{\bX}, \hat{\bbeta}^{(\tau)})$ as
\begin{align*}
    f(\tilde{\bX}, \hat{\bbeta}^{(\tau)}) &= \alpha^{(\tau)}\kappa(\tilde{\bX}, \bX)(\bK +\lambda\bI)^{-1}\left(\by - \by^{(\tau-1)} \right) + \kappa(\tilde{\bX}, \bX)(\bK +\lambda\bI)^{-1}\by^{(\tau-1)} \\
    &= \alpha^{(\tau)}\frac{\partial}{\partial \alpha^{(\tau)}}f(\tilde{\bX}, \hat{\bbeta}^{(\tau)}) + \kappa(\tilde{\bX}, \bX)(\bK +\lambda\bI)^{-1}\by^{(\tau-1)},
\end{align*}
and set $\frac{\partial}{ \partial \alpha^{(\tau)}}\L(\alpha^{(\tau)}, \lambda) = 0$, we can solve as follows
\begin{align*}
    \left(\partial f^{(\tau)}\right)^\trans \tilde{\by} - \left(\partial f^{(\tau)}\right)^\trans \kappa(\tilde{\bX}, \bX)(\bK +\lambda\bI)^{-1}\by^{(\tau-1)} &= \alpha^{(\tau)} \left(\partial f^{(\tau)}\right)^\trans \left(\partial f^{(\tau)}\right) \\
    &= \alpha^{(\tau)} \norm{\partial f^{(\tau)}}^2,
\end{align*}
where we use the notation $\partial f^{(\tau)} \eqdef \frac{\partial}{\partial \alpha^{(\tau)}}f(\tilde{\bX}, \hat{\bbeta}^{(\tau)})$ for brevity. Now since
\begin{align*}
    -\kappa(\tilde{\bX}, \bX)(\bK +\lambda\bI)^{-1}\by^{(\tau-1)} &= \kappa(\tilde{\bX}, \bX)(\bK +\lambda\bI)^{-1}(\by - \by^{(\tau-1)}) - \kappa(\tilde{\bX}, \bX)(\bK +\lambda\bI)^{-1}\by,
\end{align*}
we can finalize the proof with
\begin{align*}
    \alpha^{\star (\tau)} &= \frac{\left(\frac{\partial}{\partial \alpha} f(\tilde{\bX}, \hat{\bbeta}^{(\tau)})\right)^\trans \left(\tilde{\by} - \tilde{\by}^{(1)}\right)}{\norm{\frac{\partial}{\partial \alpha} f(\tilde{\bX}, \hat{\bbeta}^{(\tau)})}^2} + 1,
\end{align*}
and noting that $\frac{\partial}{\partial \alpha^{(\tau)}} f(\tilde{\bX}, \hat{\bbeta}^{(\tau)}) = \tilde{\by}^{(1)} - \tilde{\by}^{(\tau)}_{\alpha=0}$.
\end{proof}

Note, in the following we state and prove a slightly more general result than Theorem \ref{thm:kernel_limit}.
\begin{theorem}
Let $\by^{(\tau)}, \hat{\bbeta}^{(\tau)}$, and $f(\cdot, \hat{\bbeta}^{(\tau)})$ be defined as above, and $\alpha \in [0,1]$, then the following limits hold
\begin{align*}
    \by^{(\infty)} &\eqdef \lim_{\tau \to \infty} \by^{(\tau)} = \alpha \bK\left(\alpha\bK + \lambda\bI_n \right)^{-1}\by \\
    f(\bx, \hat{\bbeta}^{(\infty)}) &\eqdef \lim_{\tau \to \infty} f(\bx, \hat{\bbeta}^{(\tau)}) = \alpha \kappa(\bx, \bX)^\trans(\bK + \lambda\bI_n)^\inv \left(\bI_n + (1-\alpha)\bK(\alpha\bK + \lambda\bI_n)^\inv \right)\by
\end{align*}
and if $\alpha > 0$, then
\begin{align*}
    \by^{(\infty)} &= \bK\left(\bK + \frac{\lambda}{\alpha}\bI_n \right)^{-1} \by \\
    f(\bx, \hat{\bbeta}^{(\infty)}) &= \alpha f(\bx, \hat{\bbeta}^{(1)}) + (1-\alpha)f(\bx, \hat{\bgamma}^{(\infty)})
\end{align*}
where \eqref{eq:Y_limit} corresponds to \emph{classical} kernel ridge regression with amplified regularization parameter $\frac{\lambda}{\alpha}$, and we let $\hat{\bgamma}^{(\infty)}$ denote the kernel ridge regression parameter associated with solving another kernel ridge regression on the targets $\by^{(\infty)}$ with regularization parameter $\lambda$. Furthermore, the convergence $\lim_{\tau \to \infty} \by^{(\tau)}$ is of linear rate.
\end{theorem}
\begin{proof}
By \eqref{eq:kernel_with_SVD} we have that $\bK(\bK + \lambda\bI_n)^{-1} = \bV \bD \left(\bD + \lambda\bI_n\right)^{-1}\bV^\trans$ where $\lambda > 0$, $\bD$ is positive diagonal and $\bV$ is orthogonal, and hence, the eigenvalues of $\bK(\bK + \lambda\bI_n)^{-1}$ are all smaller than $1$ in absolute value, and thus $(1-\alpha)^{\tau-1}\left(\bK(\bK + \lambda\bI_n)^\inv \right)^{\tau}$ converge to the zero-matrix when $\tau \to \infty$. Thus, using the limit for a geometric series of matrices we get that
\begin{align*}
    \lim_{\tau \to \infty} \by^{(\tau)} &= \left( \frac{\alpha}{1-\alpha} \sum_{i=1}^\infty \left((1-\alpha) \bK(\bK + \lambda\bI_n)^\inv \right)^i \right)\by \\
    &= \frac{\alpha}{1-\alpha}(1-\alpha)\bK(\bK + \lambda\bI_n)^\inv (\bI_n - (1-\alpha) \bK(\bK + \lambda\bI_n)^\inv)^\inv \by \\
    &= \alpha \bK \left(\alpha\bK + \lambda\bI_n \right)^\inv \by.
\end{align*}
If $\alpha > 0$, the remaining result for $\lim_{\tau \to \infty} \by^{(\tau)}$ follows directly. Now, by inserting $\by^{(\infty)}$ and manipulating the result, we get that
\begin{align*}
    f(\bx, \bbeta^{(\infty)}) &= \kappa(\bx, \bX)^\trans(\bK + \lambda\bI_n)^\inv\left(\alpha \by + (1-\alpha)\by^{(\infty)} \right) \\
    &= \kappa(\bx, \bX)^\trans(\bK + \lambda\bI_n)^\inv\left(\alpha \bI_n + (1-\alpha)\alpha \bK\left(\alpha\bK + \lambda\bI_n \right)^{-1} \right)\by \\
    &= \alpha\kappa(\bx, \bX)^\trans(\bK + \lambda\bI_n)^\inv\left(\bI_n + (1-\alpha)\bK\left(\alpha\bK + \lambda\bI_n \right)^{-1} \right)\by,
\end{align*}
and if $\alpha > 0$, then
\begin{align*}
    f(\bx, \bbeta^{(\infty)}) &= \alpha f(\bx, \bbeta^{(1)}) + (1-\alpha)\kappa(\bx, \bX)^\trans(\bK + \lambda\bI_n)^\inv\bK\left(\bK + \frac{\lambda}{\alpha}\bI_n \right)^{-1}\by \\
    &= \alpha f(\bx, \bbeta^{(1)}) + (1-\alpha)f(\bx, \hat{\bgamma}^{(\infty)}),
\end{align*}
where we let $\hat{\bgamma}^{(\infty)}$ denote the kernel ridge regression parameter associated with the classical kernel ridge regression problem on the targets $\by^{(\infty)}$ with regularization parameter $\lambda$.

Finally, denote by $\bC \eqdef (1-\alpha) \bK(\bK + \lambda\bI_n)^\inv$, then we have that
\begin{align*}
    \bE(t) &\eqdef \sum_{i=1}^{t} \bC^i - \sum_{i=1}^{\infty} \bC^i =  \bC^{t+1}\left(\bC - \bI_n\right)^{-1},
\end{align*}
and thus for an additional $s$ steps we have $\bE(t + s) = \bC^{t+s+1}(\bC - \bI_n)^{-1} = \bC^{s}\bE(t)$. Hence, the convergence is of linear rate as claimed.
\end{proof}

%% file: experiments.tex
\section{Experiments}\label{sec:experiments}
In the following we show empirical results of performing a simple self-distillation procedure with deep neural networks with varying choices of $\alpha$ to investigate the large scale effects. The experiments are adapted from \citet{mobahi2020selfdistillation} with the additional introduction of the $\alpha$-parameter. For stronger baselines of the possible performance gains from self-distillation see e.g. \citet{furlanello2018born, Tian2020RethinkingNeed, Ahn2019VariationalTransfer, Yang2018KnowledgeStudents}.
The following sections provide additional details to that of Section \ref{sec:optimal_alpha_DL}.

\subsection{Experimental Setup}
We perform self-distillation with ResNet-50 \citep{He2016DeepRecognition} networks on CIFAR-10 \citep{Krizhevsky2009LearningImages}, with minor pre-processing and augmentations.\footnote{Training: We randomly flip an image horizontally with probability $\frac{1}{2}$, followed by a random $32 \times 32$ crop of the $40 \times 40$ zero padded image. Finally we normalize the image to have mean $0$ and standard deviation $1$. Validation: We normalize the image with the empirical mean and standard deviation from the training data.} The model is initialized randomly at each step\footnote{Note, we initialize the models equally across all $\alpha$ for one experiment, but alter the seed for initialization between experiments.} and trained as described in Section \ref{sec:optimal_alpha_DL} with either estimated optimal parameters, $\hat{\alpha}^{(\tau)}$, or fixed $\alpha$ for all steps. We use Adam optimizer with a learning rate of $10^{-4}$, $\ell_2$ regularization with regularization coefficient $10^{-4}$, and train on the full $50000$ training images and validate our generalization performance on the $10000$ test images. We use the weights from the last step of optimization at each distillation step for the next step, irrespective of whether a better model occurred earlier in the training. Our models are trained for a fixed $75$ epochs, which does not allow our models to overfit the training data, which is important for our models to be suitable for distillation procedures \citep{Dong2019DistillationNetwork}. The experiments are performed on a single Nvidia Tesla V100 16GB GPU with the PyTorch Lightning framework \citep{Falcon2019PyTorchLightning}.

\subsection{Results}
We repeat our experiment $4$ times and illustrate the mean, minimum and maximum at each distillation step in Figure \ref{fig:optimal_resnet50_by_step}. Each experiment is $11$ chains of distillation steps, corresponding to $\alpha \in \{0.0, 0.1, \dots, 0.9\}$, with the first model initialized identically across all chains. The accuracy reported at the $\tau$'th step is based on comparing the training and validation predictions, $\bY^{(\tau)}$ and $f(\tilde{\bX}, \hat{\bbeta}^{(\tau)})$ with the original training and validation targets; $\bY$ and $\tilde{\bY}$.

\begin{figure}[htbp]
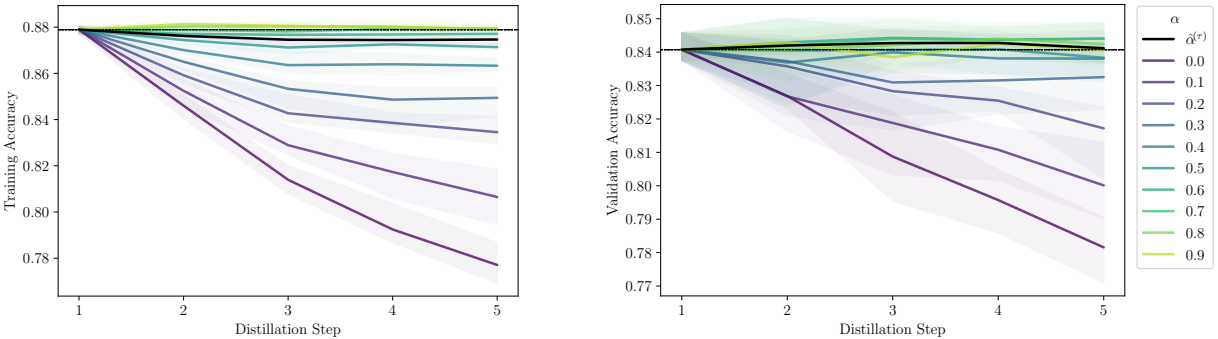

    \centering
    \begin{subfigure}[b]{0.44\linewidth}
        \centering
        \includegraphics[width=\linewidth]{figures/resnet50_optimal_train.pdf}
    \end{subfigure}
    \hfill
    \begin{subfigure}[b]{0.515\linewidth}
        \centering
        \includegraphics[width=\linewidth]{figures/resnet50_optimal_val.pdf}
    \end{subfigure}
    \caption{(Identical to Figure \ref{fig:optimal_resnet50_by_step}) Training and validation accuracy for five distillation steps with ResNet-50 models on CIFAR-10. Comparing fixed $\alpha^{(t)}$ for $t=2\dots,\tau$ and estimating optimal weight with $\hat{\alpha}^{(t)}$ at each step. The experiment is repeated four times and the mean (and max/min in shaded) is reported.}
\end{figure}

The theory introduced in Section \ref{sec:main_results} suggests that self-distillation corresponds to a progressively amplified regularization of the solution, and larger $\alpha$ dampens the amount of regularization imposed by the procedure more than small values of $\alpha$. Thus, for small $\alpha$ we expect the training accuracy to decrease with each distillation, but for larger $\alpha$ we might experience an increase in training accuracy, due to additional training iterations and a sufficient amount of ground-truth target information being kept in the optimization problem, which could prove beneficial. However, depending on the need for increased regularization, we expect the validation accuracy to increase for some $\alpha$, and possibly decrease for $\alpha$ values either too small or too large. The above properties are observed in Figure \ref{fig:optimal_resnet50_by_step}, where $\alpha \leq 0.4$ generally overregularize the solution, and performance drops with distillation steps. For $\alpha > 0.4$ the performance generally improves with distillation, but for $\alpha$ close to one, the gains reduce, and a suitably balanced $\alpha$ for this experiment would be in $[0.5, 0.7]$. This aligns well with the optimal $\hat{\alpha}^{(\tau)}$ estimated at approx. $0.6$ for each step of self-distillation.

%% file: ntk.tex
\section{Connection to Neural Networks}\label{sec:ntk}
This paper theoretically investigates self-distillation of kernel ridge regression models, but distillation procedures are more commonly used in a deep learning setting. However, recent research in the over-parameterized regime has shown great progress and connected wide neural networks with kernel ridge regression using the Neural Tangent Kernel (NTK) \citep{Lee2019WideDescent, Lee2020GeneralizedNetworks, Hu2019SimpleGuarantee}. The following is a brief and informal connection between kernel ridge regression and wide\footnote{Note, we refer to width as the number of hidden nodes in a fully connected neural network or channels in a convolutional neural network.} neural networks, motivating our problem setup and approach to estimate $\hat{\alpha}^{(\tau)}$ in Section \ref{sec:optimal_alpha_DL}.

Consider a neural network with scalar output $\nn(\bx, \btheta) \in \R$, where $\btheta(t) \in \R^D$ is the vector of all network parameters at training iteration $t \geq 0$, and $\bx \in \R^d$ some input. We consider the case where we use gradient descent on the MSE objective, $\L(\btheta) = \frac{1}{2}\sum_{i=1}^n (\nn(\bx_i, \btheta) - y_i)^2$ over some training dataset $\D_\mathrm{train} \subseteq \R^d \times \R$. Consider the first-order Taylor-expansion of $\nn(\bx, \btheta)$ w.r.t its parameters at initialization, $\btheta(0)$,
\begin{align}\label{eq:NTK_approx}
    \nn(\bx, \btheta) \approx \nn(\bx, \btheta(0)) + \langle\nabla_{\btheta} \nn(\bx, \btheta(0)), \btheta - \btheta(0)\rangle
\end{align}
where $\nn(\bx, \btheta(0))$ and $\nabla_{\btheta} \nn(\bx, \btheta(0))$ are constants w.r.t. $\btheta$. For sufficiently wide networks, \eqref{eq:NTK_approx} holds, and we say that we are in the \emph{(NTK) regime} \citep{Arora2019OnNet, Lee2019WideDescent}. Now, let $\varphi(\bx) \eqdef \nabla_{\btheta} \nn(\bx, \btheta(0))$ for any $\bx \in \R^d$, and denote the random kernel $\kappa(\bx_i, \bx_j) \eqdef \langle\varphi(\bx_i), \varphi(\bx_j)\rangle$ for any $\bx_i, \bx_j \in \R^d$ \citep{Jacot2018NeuralNetworks}. For sufficiently wide networks, the random kernel converges to a deterministic kernel, and since the r.h.s. of \eqref{eq:NTK_approx} is linear, one can show that minimizing $\L$ with gradient descent leads to the solution of the kernel regression problem, with the NTK; $\bx \mapsto \kappa(\bx, \bX)^\trans \kappa(\bX, \bX)^\inv \by$, where $\bX \in \R^{n \times d}$ is the matrix of training inputs, and $\by \in \R^n$ the vector of training targets \citep{Arora2019OnNet, Lee2019WideDescent}. It has been shown that when minimizing the $\ell_2$-regularized MSE loss, the solution becomes the kernel ridge regression solution \citep{Lee2020GeneralizedNetworks}.

The connections between neural networks and kernel ridge regressions in knowledge distillation settings have, to the best of our knowledge, not been explicitly investigated yet, but we hope that the results of this paper will improve the understanding of self-distillation of neural networks once such a connection is made rigorously.

%% file: constrained.tex
\section{Connections to Constrained Optimization Problem}\label{app:mobahi}
The setup investigated in this paper is the unconstrained optimization problem presented in \eqref{eq:unconstrained_optimization}, but some of the results can easily be extended to a constrained optimization problem, with a general regularization functional in Hilbert spaces, namely the natural extension of the setup proposed by \citet{mobahi2020selfdistillation}. For the rest of this section we assume $\alpha^{(2)} = \dots = \alpha^{(\tau)} = \alpha$. \citet{mobahi2020selfdistillation} propose to solve the problem
\begin{align}\label{eq:mohabi}
    f^{(\tau)} \eqdef \argmin_{f \in \F} \int_\mathcal{X} \int_\mathcal{X} u(\bm{x}, \bm{x}') f(\bm{x}) f(\bm{x}') d\bm{x} d\bm{x}' \quad \text{s.t.} \quad \frac{1}{N} \sum_{n=1}^{N} \left(f(\bm{x}_n) - y_n \right)^2 \leq \varepsilon,
\end{align}
where $\varepsilon > 0$ is a desired loss tolerance, $\tau \geq 1$, $f^{(0)} (\bm{x}_n) = y_n$ for $n=1, \dots, N$,  and $u$ being symmetric and such that $\forall f \in \mathcal{F}$\footnote{For a given $u$ the function space $\mathcal{F}$ is the space of functions $f$ for which the double integral in \eqref{eq:mohabi} is bounded.} the double integral is greater than or equal to $0$ with equality only when $f(x) = 0$. See \citet{mobahi2020selfdistillation} for details.

The natural extension of this problem is to include ground-truth labels, and solve the weighted problem
\begin{align}\label{eq:mohabi_extension}
    &f^{(\tau)} = \argmin_{f \in \F} \int_\mathcal{X} \int_\mathcal{X} u(\bm{x}, \bm{x}') f(\bm{x}) f(\bm{x}') d\bm{x} d\bm{x}' \\ \text{s.t.} \quad &\frac{\alpha}{N} \sum_{n=1}^{N} \left(f(\bm{x}_n) - y_n \right)^2 + \frac{1-\alpha}{N}\sum_{n=1}^{N} \left( f(\bm{x}_n) - f^{(\tau-1)}(\bm{x}_n) \right)^2 \leq \varepsilon,
\end{align}
for $\tau \geq 1$, where $\alpha \in [0,1]$ and $f^{(0)} (\bm{x}_n) = y_n$ for $n=1, \dots, N$. In \citet{mobahi2020selfdistillation}, $\alpha = 0$, and this problem completely ignores the ground truth data after the first model fit, and it is easy to see that consecutive self-fits will be penalized increasingly stronger, and eventually collapse to zero, whenever $\frac{1}{N} \sum_{n=1}^{N} \left(f(\bm{x}_n) - y_n \right)^2 \leq \varepsilon$. The case $\alpha = 1$, corresponds to fitting to the ground-truth at each iteration, and do not benefit from distillation, and thus is without interest here.

\subsection{Collapsing and converging conditions}\label{sec:collapsing_converging}
The regularization functional of \eqref{eq:mohabi_extension}, is clearly minimized by $f_t(\bm{x}) = 0$, but in order for this to be a solution for some $\tau \geq 1$, it must hold that
\begin{align*}
    \frac{\alpha}{N}\norm{\bm{y}}_2^2 + \frac{1-\alpha}{N}\norm{\bm{y}^{(\tau-1)}}_2^2 \leq \varepsilon,
\end{align*}
where we use the notation that $\bm{y}^{(\tau)} = (f^{(\tau)}(\bm{x}_1), \dots, f^{(\tau)}(\bm{x}_N))^T$ and $\bm{y} = (y_1, \dots, y_N)$. For $\tau = 1$, this amounts to $\frac{1}{N}\norm{\bm{y}}_2^2 \leq \varepsilon$, and for $\tau > 1$
\begin{align}\label{eq:collapse_condition}
    \frac{1-\alpha}{N}\norm{\bm{y}^{(\tau-1)}}_2^2 \leq \varepsilon - \frac{\alpha}{N}\norm{\bm{y}}_2^2.
\end{align}
But since the l.h.s. is non-negative, it is required that $\frac{\alpha}{N}\norm{\bm{y}}_2^2 \leq \varepsilon$ in order for $f^{(\tau)}(\bm{x}) = 0$ to be a solution. Hence, we can construct the following settings, that determine the behavior of the solutions:
\begin{enumerate}
    \item $\frac{1}{N}\norm{\bm{y}}_2^2 \in \left[0, \varepsilon\right] \implies \norm{\bm{y}^{(\tau)}}_2 = 0 ~\forall \tau \geq 1,$ \hfill (Collapsed solution)
    \item $\frac{1}{N}\norm{\bm{y}}_2^2 \in \left(\varepsilon, \frac{\varepsilon}{\alpha}\right] \implies \exists \underline{\tau} \geq 1$ such that $\begin{cases} \norm{\bm{y}^{(\tau)}}_2 > 0 &\forall \tau < \underline{\tau}, \\ \norm{\bm{y}^{(\tau)}}_2 = 0 &\forall \tau \geq \underline{\tau},\end{cases}$ \hfill (Converging to collapsed solution)
    \item $\frac{1}{N}\norm{\bm{y}}_2^2 \in \left(\frac{\varepsilon}{\alpha}, \infty\right) \implies \norm{\bm{y}^{(\tau)}}_2 > 0 ~\forall \tau \geq 1$. \hfill (Converging to non-collapsed solution)
\end{enumerate}
If we let $\alpha \to 0$ the interval $\left(\varepsilon, \frac{\varepsilon}{\alpha}\right]$ effectively becomes $\left(\varepsilon, \infty\right)$, and any solution will collapse at some point \citep{mobahi2020selfdistillation}. Analogously, if we let $\alpha \to 1$, the interval $(\varepsilon, \frac{\varepsilon}{\alpha}]$ effectively becomes empty, and all non-collapsed solutions will converge to a non-zero solution. Hence, if $\alpha > 0$, one can obtain non-collapsing convergence with infinite iterations. Furthermore, if we let $\varepsilon \to 0$, then $[0,\varepsilon]$ and $(\varepsilon, \frac{\varepsilon}{\alpha}]$ will practically collapse to empty intervals, and we will always obtain convergence to non-collapsing solutions, which will correspond to an interpolating solution.
For the remainder we assume $\alpha \in (0, 1)$, since the boundary cases are covered in \citet{mobahi2020selfdistillation} ($\alpha = 0$) or is without interest ($\alpha = 1$). Furthermore, we assume that $\norm{\by}_2 > \sqrt{N\varepsilon}$ to avoid a collapsed solution from the beginning. Utilizing the Karush-Kuhn-Tucker (KKT) conditions for this problem, we can rephrase our optimization problem as
\begin{align}\label{eq:problem_mobahi}
    f^{(\tau)} &= \argmin_{f \in \F} \frac{\alpha}{N}\sum_{n=1}^{N} \left(f(\bm{x}_n) - y_n \right)^2 + \frac{1-\alpha}{N}\sum_{n=1}^{N} \left( f(\bm{x}_n) - f^{(\tau-1)}(\bm{x}_n) \right)^2 \nonumber \\ 
    &+ \lambda_\tau \int_\mathcal{X} \int_\mathcal{X} u(\bm{x}, \bm{x}') f(\bm{x}) f(\bm{x}') d\bm{x} d\bm{x}',
\end{align}
where $\lambda_{\tau} \geq 0$. For suitably chosen $\lambda_{\tau}$, one can show that $f^{(\tau)}$ is an optimal solution to our problem\footnote{See \citet{mobahi2020selfdistillation} for a detailed argument.}.

\subsection{Extending our results}
By direct calculations similar to those of \citet{mobahi2020selfdistillation} one can obtain the closed form solution of \eqref{eq:problem_mobahi}, but first we will repeat some definitions from \citet{mobahi2020selfdistillation}. Let the Green's Function $g(x,t)$ be such that $\int_{\mathcal{X}} u(x, x') g(x', t) dx' = \delta(x - t)$, where $\delta$ is the Dirac delta, and let $[\bG]_{j,k} = \frac{1}{N} g(\bx_j, \bx_k)$ and $[\bg(\bx)]_k = \frac{1}{N} g(\bx, \bx_k)$, where $\bG$ is a matrix and $\bg(\bx)$ a vector dependent on $\bx$.
Now we can present the proposition.
\begin{proposition}\label{prop:solution}
For any $\tau \geq 1$, the problem \eqref{eq:problem_mobahi} has a solution of the form
\begin{align*}
    \by^{(\tau)} = \bg(\bx)^\trans \left(\bG +  \lambda_{\tau}\bI\right)^{-1}(\alpha \by + (1-\alpha)\by^{(\tau-1)}),
\end{align*}
where $\by^{(0)} \eqdef \by$.
\end{proposition}

Since $\bG$ is positive semi-definite, we can decompose it as $\bG = \bV \bD \bV^\trans$. Define $\bB^{(0)} \eqdef \bI$, then for $\tau \geq 1$, we have that $\by^{(\tau)} = \bV \bB^{(\tau)} \bV^\trans \by$, where we set
\begin{align*}
    \bB^{(\tau)} &\eqdef \frac{\alpha}{1-\alpha}\sum_{i=1}^{\tau-1} \left(1-\alpha\right)^{\tau-i} \prod_{j=i}^{\tau-1} \bA^{(j+1)} + \left(1-\alpha\right)^{\tau-1}\prod_{j=1}^\tau \bA^{(j)}, \\
    \bA^{(\tau)} &\eqdef \bD(\bD + \lambda_{\tau} \bI)^{\inv}.
\end{align*}

By equivalent calculations as those of Lemma \ref{lem:recurrent_B}, we have that
\begin{align*}
    \bB^{(\tau)} = \bA^{(\tau)}((1-\alpha)\bB^{(\tau-1)} + \alpha \bI_N) \quad \forall \tau \geq 1,
\end{align*}
and we can now formulate the following theorem, similar to Theorem \ref{thm:B_sparsification}.

\begin{theorem}
For any pair of diagonals of $\bD$, i.e. $d_k$ and $d_j$, where $d_k > d_j$, we have that for all $\tau \geq 1$ and for $\alpha \in (0, 1)$, then
\begin{align}
    &\mathrm{sign}\left(\frac{[\bB^{(\tau)}]_{k,k}}{[\bB^{(\tau)}]_{j,j}} - \frac{[\bB^{(\tau-1)}]_{k,k}}{[\bB^{(\tau-1)}]_{j,j}} \right) \\
    &= \mathrm{sign}\left(\left( \left(\frac{[\bB^{(\tau-1)}]_{k,k}}{[\bB^{(\tau-1)}]_{j,j}} - \frac{[\bA^{(\tau)}]_{k,k}}{[\bA^{(\tau)}]_{j,j}} \right) \frac{[\bA^{(\tau)}]_{j,j}}{[\bB^{(\tau-1)}]_{k,k} ([\bA^{(\tau)}]_{k,k} - [\bA^{(\tau)}]_{j,j})} + 1\right)^\inv - \alpha \right).
\end{align}
\end{theorem}
\begin{proof}
The proof follows by analogous calculations to the proof of Theorem \ref{thm:B_sparsification}.
\end{proof}
For the case $\alpha = 1$, the results is identical to Theorem \ref{thm:B_sparsification}, since no distillation is actually performed. However, the case $\alpha = 0$ is more involved and we refer the reader to \citet{mobahi2020selfdistillation} for the treatment of this case, since our setup is identical to that of \citet{mobahi2020selfdistillation} when $\alpha = 0$.

Finally, due to the dependency of $\lambda_{\tau}$ on the solution from the previous step of distillation in this constrained optimization problem, we are unable to obtain a simple recurrent expression as \eqref{eq:recurrent_y} and a limiting solution as in Theorem \ref{thm:kernel_limit}. However, by the results in Section \ref{sec:collapsing_converging}, we can determine, from the norm of the targets, whether the solution collapses or not.